\documentclass[twoside]{article}

\usepackage[round]{natbib}

\usepackage[colorlinks=false,plainpages=false,pdfborder={0 0 0},backref=false]{hyperref}
\usepackage[caption=false]{subfig}
\usepackage{xspace}
\usepackage{enumitem}

\usepackage{graphicx}
\graphicspath{{./figures/}}
\DeclareGraphicsExtensions{.pdf,.jpg,.jpeg,.png}

\usepackage{amsmath}
\usepackage{amssymb}
\usepackage{amsthm}
\usepackage{bm}
\usepackage{dsfont}

\usepackage[ruled,linesnumbered,vlined]{algorithm2e}
\usepackage{algpseudocode} 

\usepackage{thmtools, thm-restate}

\usepackage{tikz}
\usetikzlibrary{arrows}
\usepackage{pgfplots}
\pgfplotsset{compat=1.5}

\usepackage[accepted]{aistats2019}

\def\tikzrobot(#1,#2,#3){
    \draw (#1, #2) circle [radius=0.25cm];
    \draw [black] (#1, #2) -- +(#3:0.25cm);
}

\def\tikzdashedrobot(#1,#2,#3){
    \draw[dashed] (#1, #2) circle [radius=0.25cm];
    \draw[black] (#1, #2) -- +(#3:0.25cm);
}

\def\tikzcross(#1,#2,#3){
    \draw (#1-#3, #2-#3) -- (#1+#3, #2+#3);
    \draw (#1-#3, #2+#3) -- (#1+#3, #2-#3);
}

\newcommand*{\argmax}{\operatornamewithlimits{argmax}}

\newcommand*{\field}[1]{\mathbb{\MakeUppercase{#1}}}		
\newcommand*{\set}[1]{{\mathcal{\MakeUppercase{#1}}}}			
\newcommand*{\norm}[1]{\lVert #1 \rVert}	
\newcommand*{\flexnorm}[1]{\left\lVert #1 \right\rVert}
\newcommand*{\inner}[1]{\langle #1 \rangle}	
\newcommand*{\collection}[1]{{\mathfrak{\MakeUppercase{#1}}}} 

\newcommand*{\functional}[1]{{\MakeUppercase{#1}}}

\newcommand*{\R}{\field{R}} 
\newcommand*{\N}{\field{N}} 

\renewcommand{\vec}[1]{{\boldsymbol{\mathbf{#1}}}}
\newcommand*{\mat}[1]{\vec{\MakeUppercase{#1}}}
\newcommand*{\eye}{\mat{I}}							
\newcommand*{\transpose}{\mathsf{T}}
\newcommand*{\tr}{{\operatorname{tr}}}						

\newcommand*{\dataset}{\set{D}} 				
\newcommand*{\observation}{y} 					
\newcommand*{\observations}{\vec{\observation}}
\newcommand*{\obsNoise}{\nu}

\newcommand*{\gp}{\operatorname{GP}}
\newcommand*{\gpMean}{\mu}

\newcommand*{\mutinfo}[1]{I(#1)}				

\newcommand*{\slocation}{x} 
\newcommand*{\location}{\vec{\slocation}} 
\newcommand*{\locDomain}{\set{X}} 
\newcommand*{\locSet}{\set{Q}} 

\newcommand*{\dimension}{d}
\newcommand*{\locDim}{\dimension}

\newcommand*{\measure}[1]{\MakeUppercase{#1}}
\newcommand*{\expectation}{\mathbb{E}}

\newcommand*{\Pspace}{\set{P}} 					
\newcommand*{\pMeasure}{\measure{P}}
\newcommand*{\pDensity}{p}
\newcommand*{\prob}[1]{\mathbb{P}\left\lbrace #1 \right\rbrace}
\newcommand*{\pSet}{\set{R}}
\newcommand*{\kl}[2]{D_\mathrm{KL}(#1||#2)}
\newcommand*{\normal}{\measure{N}}					
\newcommand*{\Dirac}{\measure{d}}
\newcommand*{\diff}{{\mathop{}\operatorname{d}}}

\newcommand*{\rv}{\xi}								

\newcommand*{\as}{\mathrm{(a.s.)}}			

\newcommand*{\filtration}{\collection{F}}

\newcommand*{\qry}{E}								
\newcommand*{\loc}{L}								
\newcommand*{\srandom}[1]{\tilde{#1}}
\newcommand*{\random}[1]{\vec{\srandom{#1}}}
\newcommand*{\uncertain}[1]{\hat{#1}}
\newcommand*{\Px}[1][\location]{\hat{\pMeasure}_{#1}}
\newcommand*{\locMean}{\vec{\hat{\location}}}
\newcommand*{\locNoise}{\vec{\epsilon}}
\newcommand*{\uObsNoise}{\zeta}
\newcommand*{\kp}{\uncertain{k}}
\newcommand*{\Kp}{{\mat{\uncertain{k}}}}
\newcommand*{\error}{\rho}

\newcommand*{\Hspace}{\set{H}} 						
\newcommand*{\meanMap}{\psi}						

\newcommand*{\lspan}{{\operatorname{span}}} 
\newcommand*{\nullspace}{{\operatorname{Null}}} 
\newcommand*{\operator}[1]{\mat{#1}}
\newcommand*{\projection}{\operator{\Pi}}

\newcommand*{\af}{h} 						
\newcommand*{\Sspace}{\set{S}} 				
\newcommand*{\regret}{r}					
\newcommand*{\Regret}{\functional{R}}					
\newcommand*{\uregret}{\uncertain{\regret}}	
\newcommand*{\uRegret}{\uncertain{\Regret}}	

\newcommand*{\mig}{\gamma}

\newcommand*{\iterIdx}{t}

\newcommand*{\primIdx}{i}	
\newcommand*{\secIdx}{j}	

\newcommand*{\nIterations}{n}
\newcommand*{\nObs}{n}

\newcommand*{\nFeatures}{m}

\newcommand*{\anyspace}{\set{X}}
\newcommand*{\anotherspace}{\set{Y}}
\newcommand*{\anyset}{\set{U}}

\newcommand*{\anyelement}{x}

\newcommand*{\Lipschitz}{\ell}	
\newcommand*{\bound}{b}		

\newcommand*{\factor}{\eta}	

\newcommand*{\anyconstant}{c}

\newcommand*{\anyfunction}{u}

\newcommand*{\anydomain}{\set{W}}

\newcommand*{\map}{\operator{M}} 

\newcommand*{\iid}{i.i.d.\xspace}

\declaretheorem[name=Theorem]{theorem}
\declaretheorem[numberlike=theorem,name=Proposition]{proposition}
\declaretheorem[numberlike=theorem,name=Lemma]{lemma}
\declaretheorem[numberlike=theorem,name=Definition]{definition}

\providecommand{\varitem}{} 
\makeatletter
\newenvironment{statements}[1]
 {\renewcommand\varitem[1]{\item[\textbf{#1\arabic{enumi}\rlap{$##1$}.}]%
    \edef\@currentlabel{#1\arabic{enumi}{$##1$}}}%
  \enumerate[label=\textbf{#1\arabic*.}, ref=#1\arabic*]}
 {\endenumerate}
\makeatother

\newenvironment{axioms}{\statements{A}}{\endstatements}

\begin{document}
%
\runningauthor{Rafael Oliveira, Lionel Ott and Fabio Ramos}

\twocolumn[

\aistatstitle{Bayesian optimisation under uncertain inputs}
\newcommand{\affmark}[1]{\textsuperscript{\textnormal{#1}}}

\newcommand{\email}[1]{\href{mailto:#1}{#1}}

\aistatsauthor{Rafael Oliveira \And Lionel Ott \And  Fabio Ramos}
\aistatsaddress{\email{rafael.oliveira@sydney.edu.au}\\The University of Sydney \And \email{lionel.ott@sydney.edu.au}\\The University of Sydney \And \email{fabio.ramos@sydney.edu.au}\\The University of Sydney \& NVIDIA}
]

\begin{abstract}
Bayesian optimisation (BO) has been a successful approach to optimise functions which are expensive to evaluate and whose observations are noisy. Classical BO algorithms, however, do not account for errors about the location where observations are taken, which is a common issue in problems with physical components. In these cases, the estimation of the actual query location is also subject to uncertainty. In this context, we propose an upper confidence bound (UCB) algorithm for BO problems where both the outcome of a query and the true query location are uncertain. The algorithm employs a Gaussian process model that takes probability distributions as inputs. Theoretical results are provided for both the proposed algorithm and a conventional UCB approach within the uncertain-inputs setting. Finally, we evaluate each method's performance experimentally, comparing them to other input noise aware BO approaches on simulated scenarios involving synthetic and real data.
\end{abstract}

\section{Introduction}
Bayesian optimisation (BO) \citep{Brochu2010} is a technique to find the global optimum of functions that are unknown, expensive to evaluate, and whose output observations are possibly noisy. In this sense, BO has been applied across different fields to a wide class of problems, including hyper-parameter tuning \citep{Snoek2012}, policy search \citep{Wilson2014}, environmental monitoring \citep{Marchant2012}, robotic grasping \citep{Nogueira2016}, etc. Although taking into account that we might have a noisy observation of the function's output value, conventional BO approaches assume that the function has been sampled precisely at the specified query location within the given search space. While this is true for many applications of BO, there are certain problems, especially in areas of robotics and process control, in which this assumption typically does not hold.

As an illustration, consider a problem where we are interested in finding the peak of an environmental process $f(\location)$ over a region $\Sspace \subset \R^\locDim$. To this end, we send a mobile robot to different target locations $\location_\iterIdx \in \Sspace$ to observe the process. Unfortunately, due to localisation uncertainty and motion control errors, \emph{execution noise} prevents the robot from reaching the planned target location exactly. Instead, after each query, the robot provides us with an estimate of its actual location $\random{\location}_\iterIdx$ via a probability distribution $\pMeasure^\loc_\iterIdx$, which takes into account \emph{localisation noise}, as depicted in \autoref{fig:querying}. From each query, we obtain a noisy observation of the environmental process $\observation_\iterIdx = f(\random{\location}_\iterIdx) + \uObsNoise_\iterIdx$, where $\uObsNoise_\iterIdx$ is an independent noise term. In this scenario, both the function inputs $\random{\location}_\iterIdx$, i.e. query locations, and outputs $f(\random{\location}_\iterIdx)$ are not directly observable.

This paper investigates optimisation problems where input noise affects both the execution of a query and the estimation of its true location. In particular, we analyse the standard BO approach when employing the improved Gaussian process upper-confidence bound (IGP-UCB) \citep{Chowdhury2017} algorithm under input noise, and we propose the \emph{uncertain-inputs Gaussian process upper confidence bound} (uGP-UCB) algorithm. The latter is equipped with a GP model that takes probability distributions as inputs in a similar framework to \citet{Oliveira2017isrr}. We apply kernel embeddings techniques \citep{Muandet2016} to obtain the first theoretical results for BO under uncertain inputs, bounding the regret of both uGP-UCB and IGP-UCB. In addition, experiments provide empirical performance evaluations of different BO approaches to problems involving input noise.

\begin{figure}[t]
\begin{center}
\begin{tikzpicture}
	\fill[rotate around={60:(0,0)},blue!50,opacity=0.5] (0,0) ellipse (0.5cm and 1cm); 
	\node (probstart) at (-0.5cm,0.3cm) {$\pMeasure^\loc_{t-1}$};
    \tikzrobot(0,0,60); 
	\fill[red!30,opacity=0.75] (2,2) ellipse (1cm and 0.75cm); 
   	\node (probexec) at (1.5,2.25) {$\pMeasure^\qry_{\location}$};
    \draw (0,0) .. controls (1,2) .. (2,2); 
    \tikzcross(2,2,0.1); 
	\node (target) at (2,1.6) {$\location_\iterIdx$};

	\fill[rotate around={30:(2.5,2.5)},blue!50,opacity=0.5] (2.5,2.5) ellipse (0.5cm and 1.2cm); 
	\tikzrobot(2.5,2.5,30);
	\node (probend) at (2.1,3.25) {$\pMeasure^\loc_{t}$};

    \tikzdashedrobot(0.5,-0.25,30); 
    \node (truestart) at (0.5,-0.25) {}; 
   	\node (truestartlabel) at (1.25,-0.25) {$\random{\location}_{t-1}$};
    \tikzdashedrobot(2.75,2,60);
    \node (trueend) at (2.75,2) {}; 
    \draw[dashed] (truestart) to[out=30, in=240, looseness=1] (trueend);
	\node (truequery) at (3.25,2) {$\random{\location}_\iterIdx$};
\end{tikzpicture}
\end{center}
\caption{At time $t-1$, the robot is estimated to be at some $\random{\location}_{t-1} \sim \pMeasure^\loc_{t-1}$. It is then sent to target location $\location_\iterIdx$. However, due to uncertainty in the query execution, represented by $\pMeasure_{\location}^\qry$, the robot actually ends up at another location $\random{\location}_\iterIdx$, whose belief distribution, according to the localisation system, is represented by $\pMeasure_\iterIdx^\loc$. The robot's true locations and true path are indicated by the dashed lines.}
\label{fig:querying}
\end{figure}
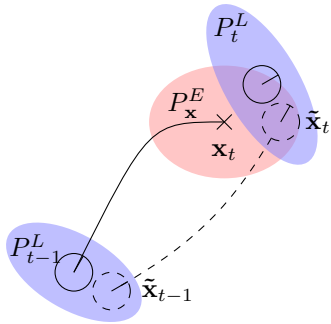

\section{Related work}
\label{sec:related}
Recently several BO approaches that deal with problems where the execution of queries to an objective function is affected by uncertainty have been proposed. \citet{Nogueira2016} presented a method that applies the unscented transform \citep{Wan2000} to query BO's acquisition function. By considering a stochastic query execution process, the method is able to find robust solutions to robotics problems such as grasping. Another approach to handle query uncertainty is presented in \citet{Pearce2017} to optimise stochastic simulations. In that case, query uncertainty refers to imperfect knowledge about input variates for a simulation model \citep{Lam2016}. \citeauthor{Pearce2017} apply Monte Carlo integration to marginalise out input variates that are unknown when querying BO's acquisition function. In broader terms, all of these problems can be described as optimising an integrated cost function, where one may instead use a GP prior over the integrated function \citep{Beland2017,Toscano-Palmerin2018}. Contrasted to uGP-UCB, however, the approaches mentioned above only deal with independent and identically distributed input noise and mostly offer no known theoretical guarantees. In addition, the data points in their GP datasets are only point estimates, instead of distributions as used in this paper.

Another BO approach is presented in \citet{Oliveira2017isrr}, which employed a Gaussian process (GP) model that takes probability distributions directly as inputs \citep{Girard2004,Dallaire2011}. However, \citeauthor{Oliveira2017isrr}'s method intent is to learn a model of the objective function with a robot, while minimising travelled distance, not as an optimisation framework.

Problems like the one illustrated in \autoref{fig:querying} can also be related to partially-observable Markov decision processes (POMDPs) \citep{Marchant2014,Ling2016}. This paper, however, is concerned with a general optimisation setup. 

\section{Problem formulation}
We consider an optimisation problem where an algorithm sequentially selects target locations $\location_\iterIdx$ within a compact search space $\Sspace\subset\locDomain$ at which to query a function $f:\locDomain\to\R$, seeking its global optimum. In addition, the query execution itself is a stochastic process, leading the query to be made at some $\random{\location}_\iterIdx|\location_\iterIdx\sim\pMeasure^\qry_{\location}$, instead.

How close the algorithm is to the global optimum can be measured in terms of regret. In a bandits optimisation setting, the \emph{instantaneous regret} suffered by a maximisation algorithm for a choice of target $\location_\iterIdx$ in our problem is given by:
\begin{equation}
\srandom{\regret}_\iterIdx = \max_{\location\in\Sspace} f(\location) - f(\random{\location}_\iterIdx)~.
\end{equation}
In the deterministic-inputs case, the algorithmic design goal is to minimise cumulative regret, ensuring that the algorithm eventually hits the global optimum of $f$ \citep{Srinivas2010,Bull2011}. However, as $\random{\location}_\iterIdx$ is subject to noise, one can attempt to minimise the expected regret, which is such that:
\begin{equation}
\begin{split}
\expectation[\srandom{\regret}_\iterIdx|\location_\iterIdx] &= \max_{\location\in\Sspace} f(\location) - \expectation[f(\random{\location}_\iterIdx)|\location_\iterIdx] = \error_\qry + \uregret_\iterIdx~,
\label{eq:eregret}
\end{split}
\end{equation}
where:
\begin{align}
\error_\qry &:= \max_{\location\in\Sspace} f(\location) - \max_{\location\in\Sspace}\expectation[f(\random{\location})|\location]\label{eq:qry-error}\\
\uregret_\iterIdx &:= \max_{\location\in\Sspace}\expectation[f(\random{\location})|\location] - \expectation[f(\random{\location}_\iterIdx)|\location_\iterIdx]~.\label{eq:uregret}
\end{align}
Here $\error_\qry$ is a constant, representing the difference between the maximum of the function and the maximum value any algorithm is expected to reach under the query execution uncertainty.  However, $\uregret_\iterIdx$
is controllable via the algorithm's choices of $\location_\iterIdx$ and is associated with the goal of finding:
\begin{equation}
\location^* \in \argmax_{\location \in \Sspace} \expectation[f(\random{\location})|\location] ~,
\label{eq:main_prob}
\end{equation}
which defines a target location that maximises the expected value of the function $f$ under the querying process noise. As defined, $\location^*$ minimises the expected regret to a lower bound given by $\error_\qry$ and defines an optimum location which is robust to execution noise. Therefore, we call $\uregret_\iterIdx$ the \emph{uncertain-inputs regret}. Similarly, we also define the \emph{uncertain-inputs cumulative regret} $\uRegret_\nIterations=\sum_{\iterIdx=1}^{\nIterations} \uregret_\iterIdx$. With these definitions, an algorithm whose uncertain-inputs cumulative regret $\uRegret_\nIterations$ grows sub-linearly achieves a minimum on the expected regret:
\begin{equation}
\begin{split}
\lim_{\nIterations \to \infty}\min_{\iterIdx\leq\nIterations} \expectation[\srandom{\regret}_\iterIdx|\location_t] &= \error_\qry + \lim_{\nIterations \to \infty} \min_{\iterIdx\leq\nIterations} \uregret_\iterIdx\\
&\leq \error_\qry + \limsup_{\nIterations \to \infty}\frac{\uRegret_\nIterations}{\nIterations} = \error_\qry~.
\end{split}
\end{equation}

\paragraph{Distribution assumptions:} We are assuming that the query location distribution $\pMeasure_\location^\qry$ marginalises over all other variables that could affect the querying process, such as starting points and effects from the environment that the agent is in. In addition, the true $\pMeasure_\location^\qry$ might be unknown. However, \emph{after} each query, we assume that a distribution $\pMeasure_\iterIdx^\loc$ estimating the true query location is available. These probability distributions are illustrated by the example in \autoref{fig:querying} for a robotics case.

For each $\location_\iterIdx$, the algorithm is provided with observations $\observation_t = f(\random{\location}_\iterIdx) + \uObsNoise_\iterIdx$, where $\uObsNoise_\iterIdx$ is $\sigma_\uObsNoise$-sub-Gaussian observation noise, for some $\sigma_\uObsNoise \geq 0$. Sub-Gaussian random variables can be thought of as any random variable whose tail distribution decays at least as fast as a Gaussian. Both Gaussian and bounded random variables fall in this category \citep{Boucheron2013}.

\paragraph{Regularity assumptions:} We assume $f:\locDomain\to\R$ to be an element of $\Hspace_k$, which is a reproducing kernel Hilbert space (RKHS) \citep{Scholkopf2002}.
For a given positive-definite kernel $k:\locDomain\times\locDomain\to\R$, a RKHS $\Hspace_{k}$ is a Hilbert space of functions with inner product $\inner{\cdot,\cdot}_k$ and norm $\norm{\cdot}_k = \sqrt{\inner{\cdot,\cdot}_k}$ such that $f(\location) = \inner{f,k(\cdot,\location)}_k$, for any $f\in\Hspace_k$ and any $\location\in\locDomain$.
We assume $k$ is continuous and bounded on $\locDomain\times\locDomain$, with $k(\location,\location) \leq 1, \forall \location\in\locDomain$, and that $\norm{f}_k\leq \bound$ for the objective function in \autoref{eq:main_prob}, where $\bound>0$ is known.\footnote{These assumptions are met by most of the popular kernels in BO and are common in the regret bounds literature.} When not explicitly mentioned, assume an Euclidean domain for $f$, i.e. $\locDomain\subseteq \R^\locDim$, $\locDim\in\N$.

\section{The uGP-UCB algorithm}
\label{sec:method}
This section describes a method for Bayesian optimisation under uncertain inputs. The section starts by presenting a Gaussian process that allows direct modelling of objectives defined in terms of expectations. This GP approach is then applied to derive a BO algorithm named \emph{uncertain-inputs Gaussian process upper confidence bound} (uGP-UCB), presented in the second part of this section.

\subsection{Gaussian process priors with uncertain inputs}
To extend BO to the case where query locations $\location$ are uncertain, we can redefine the objective in \autoref{eq:main_prob} as a function of the query probability distributions. Let $\Pspace$ denote the set containing all probability measures on $\locDomain \subseteq \R^\locDim$. With $f\in\Hspace_k$, we can define the map:
\begin{equation}
\begin{split}
\meanMap: \Pspace &\to \Hspace_k\\
\pMeasure &\mapsto \int_{\locDomain} k(\cdot, \location) \diff\pMeasure(\location)~.
\end{split}
\label{eq:kme}
\end{equation}
For any $\locDomain$-valued random variable $\random{\location}$ distributed according to $\pMeasure\in\Pspace$, we then have that:
\begin{equation}
\expectation_\pMeasure[f]:=\expectation[f(\random{\location})] = \inner{\meanMap_\pMeasure, f }_k, \quad \forall f \in \Hspace_k~,
\label{eq:kme-exp}
\end{equation}
where $\meanMap_\pMeasure := \meanMap(\pMeasure)$. If the kernel $k$ is characteristic, such as radial kernels \citep{Sriperumbudur2011}, $\meanMap$ is injective, defining a one-to-one relationship between measures in $\Pspace$ and elements of $\Hspace_k$. Therefore, $\meanMap$ is referred to as the mean map, and $\meanMap_\pMeasure$ as the kernel mean embedding of $\pMeasure$ \citep{Muandet2016}.

Using $\meanMap$ as defined in \autoref{eq:kme}, one can construct kernels over the set of probability measures $\Pspace$. In particular, for any $\pMeasure,\pMeasure' \in \Pspace$, we have that:
\begin{equation}
\kp(\pMeasure,\pMeasure') := \langle \meanMap_{\pMeasure}, \meanMap_{\pMeasure'} \rangle_k = \int_{\locDomain}\int_{\locDomain} k(\location,\location')\diff\pMeasure(\location)\diff\pMeasure'(\location') \label{eq:ugp-k}
\end{equation}
defines a positive-definite kernel over $\Pspace$ \citep{Muandet2012}. Notice that in this formulation, even if we have inputs representing the same random variable $\random{\location}\sim\pMeasure$, we have $\kp(\pMeasure,\pMeasure) = \inner{\meanMap_{\pMeasure},\meanMap_\pMeasure} \neq \expectation[k(\random{\location},\random{\location})]$, which is then different from other kernel formulations for models with uncertain inputs \citep{Dallaire2011}.

The kernel in \autoref{eq:ugp-k} is associated with a RKHS $\Hspace_{\kp}$ containing functions over the space of probability measures $\Pspace$.  Besides the linear kernel in \autoref{eq:ugp-k}, many other kernels on $\Pspace$ can be defined via $\meanMap$, e.g. radial kernels using $\norm{\meanMap_\pMeasure-\meanMap_{\pMeasure'}}_k$ as a metric on $\Pspace$ \citep{Muandet2012}. However, the simple kernel in \autoref{eq:ugp-k} provides us with a useful property to model the objective in \autoref{eq:main_prob}, as presented next.

\begin{lemma}[restate=expectedfunction,name={Expected function}]
\label{thr:exp-f}
Any $f \in \Hspace_k$ is continuously mapped to a corresponding $\uncertain{f} \in \Hspace_{\kp}$, which is such that:
\begin{equation}
\begin{split}
\forall \pMeasure \in \Pspace, \quad \uncertain{f}(\pMeasure) &= \expectation_\pMeasure[f]\\
\norm{\uncertain{f}}_{\kp} &= \norm{f}_k~.
\end{split}
\label{eq:exp-f-def}
\end{equation}
The mapping $f\mapsto\uncertain{f}$ constitutes an isometric isomorphism between $\Hspace_k$ and $\Hspace_{\kp}$.
\end{lemma}
\begin{proof}[Proof sketch]
The proof follows from the fact that Dirac measures $\Dirac_\location$, for $\location\in\locDomain$, are also probability measures in $\Pspace$. Since $k(\location,\location')=\kp(\Dirac_\location,\Dirac_{\location'})$, $\forall\location,\location'\in\locDomain$, we can define a bijective mapping between $\Hspace_{k}$ and $\Hspace_{\kp}$ that preserves norms. A complete proof is presented in the appendix.
\end{proof}

As a positive-definite kernel, $\kp$ defines the covariance function of a Gaussian process $\gp(0,\kp)$ modelling functions over $\Pspace$. This GP model can then be applied to learn $\uncertain{f}$ from a given set of observations $\dataset_\nObs = \{(\pMeasure_\primIdx,\observation_\primIdx)\}_{\primIdx=1}^{\nIterations}$, as in \citet{Girard2004}. Under a zero-mean GP assumption, the value of $\uncertain{f}(\pMeasure_*)$ for a given $\pMeasure_*\in\Pspace$ follows a Gaussian posterior distribution with mean and variance given by:
\begin{align}
\uncertain{\gpMean}_\nObs(\pMeasure_*) &= \vec{\kp}_\nObs(\pMeasure_*)^\transpose (\Kp_\nObs + \lambda \mat{I})^{-1} {\observations}_\nObs~, \label{eq:ugp-mean}\\
\kp_\nObs(\pMeasure,\pMeasure') &= \kp(\pMeasure,\pMeasure') - \vec{\kp}_\nObs(\pMeasure)^\transpose (\Kp_\nObs + \lambda \mat{I})^{-1}\vec{\kp}_\nObs(\pMeasure')\label{eq:posterior_kp}\\ 
\uncertain{\sigma}^2_\nObs(\pMeasure_*) &= \kp_\nObs(\pMeasure_*,\pMeasure_*)~, \label{eq:ugp-var}
\end{align}
where $\observations_\nIterations:=[\observation_1,\dots,\observation_\nIterations]^\transpose$, $\vec{\kp}_\nObs(\pMeasure_*) := [\kp(\pMeasure_*,\pMeasure_1),\dots,\kp(\pMeasure_*,\pMeasure_\nObs)]^\transpose$ and $[\Kp_\nObs]_{ij} = \kp(\pMeasure_i, \pMeasure_j)$. For a $\uncertain{f}\in\Hspace_{\kp}$, we have that $\uncertain{f}$ is generally not a sample from the GP \citep[p. 131]{Rasmussen2006}. However, we always have $\uncertain{\gpMean}_\nObs\in\Hspace_{\kp}$ by definition, allowing the GP to learn an approximation for $\uncertain{f}$. Therefore, in these equations, $\lambda\geq 0$ is simply a parameter that is not necessarily related to the true observation noise as in usual GP modelling assumptions \citep{Rasmussen2006}.

\subsection{Upper-confidence bound}
Coming back to the problem definition in \autoref{eq:main_prob}, we consider a function $\uncertain{f}:\Pspace\to\R$, such that for any $\random{\location}\sim\pMeasure, ~ \uncertain{f}(\pMeasure) = \expectation[f(\random{\location})]$. The GP model proposed in the previous section allows deriving a BO algorithm to solve the problem in \autoref{eq:main_prob}. Given a set of past observations $\dataset_{\iterIdx-1}=\{(\pMeasure_\primIdx,\observation_\primIdx)\}_{\primIdx=1}^{\iterIdx-1}$,
the following defines an upper confidence bound (UCB) acquisition function:
\begin{equation}
\af(\pMeasure|\dataset_{\iterIdx-1}) = \uncertain{\gpMean}_{\iterIdx-1}(\pMeasure) + \beta_\iterIdx {\uncertain{\sigma}_{\iterIdx-1}(\pMeasure)} ~,
\label{eq:ugp-ucb}
\end{equation}
where $\beta_\iterIdx$ is a parameter controlling the exploration-exploitation trade-off. The theoretical results in the next section will show that $\beta_\iterIdx$ can be set accordingly for $\af(\pMeasure|\dataset_{\iterIdx-1})$ to maintain a high-probability upper bound on $\uncertain{f}$.

Querying the GP model with $\location\mapsto\pMeasure_{\location}^\qry$ would allow selecting points $\location_\iterIdx$ based on an estimate for $\expectation_{\pMeasure^\qry_{\location_\iterIdx}}[f]:=\expectation[f(\random{\location}_\iterIdx)|\location_\iterIdx]$. However, in general, the true mapping $\location \mapsto \pMeasure_{\location}^\qry$ is unknown. Instead, we use a model $\location\mapsto\Px$ whose approximation error $|\expectation_{\pMeasure_{\location}^\qry}[f]-\expectation_{\Px}[f]|$ is small.

\autoref{alg:ugp-ucb} presents the uGP-UCB algorithm. Equipped with the acquisition function in \autoref{eq:ugp-ucb}, at each iteration $\iterIdx$, the algorithm selects the target location ${\location}_\iterIdx$ that maximises $\af(\Px|\dataset_{\iterIdx-1})$ (\autoref{lin:ugp-ucb-af}). 
In \autoref{lin:ugp-ucb-sample}, the function $f$ is queried at some location $\random{\location}_\iterIdx|\location_\iterIdx\sim\pMeasure_{\location_\iterIdx}^\qry$. After the query is done, the algorithm is provided with an observation $\observation_\iterIdx = f(\random{\location}_\iterIdx)+\uObsNoise_\iterIdx$ and an independent estimate for $\random{\location}_\iterIdx$ given by $\pMeasure_\iterIdx^\loc$, as described earlier. In \autoref{lin:ugp-ucb-update}, the GP model is updated with the new observation pair $(\pMeasure_\iterIdx^\loc,\observation_\iterIdx)$. This process then repeats for a given number of iterations $\nIterations$. As a result, the algorithm finishes with an estimate of the optimum location $\location^*$ given as the target location with the best estimated outcome $\location_\nIterations^*$ (\autoref{lin:ugp-ucb-final}).

\begin{algorithm}[t]
\caption{uGP-UCB}
\label{alg:ugp-ucb}
\DontPrintSemicolon
\KwIn{$\Sspace$: search space\linebreak
$\nIterations$: total number of iterations
}
\For{$\iterIdx\in\{1,\dots,\nIterations\}$}{
	$\location_\iterIdx = \underset{\location \in \Sspace}{\argmax{}} \uncertain{\gpMean}_{\iterIdx-1}(\Px) + \beta_\iterIdx {\uncertain{\sigma}_{\iterIdx-1}(\Px)}$\label{lin:ugp-ucb-af}\;
	$(\pMeasure^\loc_\iterIdx,\observation_\iterIdx) \leftarrow$ Sample $f$ at $\random{\location}_\iterIdx|\location_\iterIdx \sim\pMeasure_{\location_\iterIdx}^\qry$\label{lin:ugp-ucb-sample}\;
	$\dataset_{\iterIdx} = \dataset_{\iterIdx-1} \cup \{(\pMeasure^\loc_\iterIdx,\observation_\iterIdx)\}$\label{lin:ugp-ucb-update}\;
}
$\location_\nIterations^* = \underset{{\location_\iterIdx \in \dataset_\nIterations}}{\argmax{}} \uncertain{\gpMean}_\nIterations(\Px[\location_\iterIdx])$\label{lin:ugp-ucb-final}\;
\KwResult{$\location_\nIterations^*$}
\end{algorithm}

\section{Theoretical results}
\label{sec:analysis}
This section presents theoretical results bounding the uncertain-inputs regret of the uGP-UCB algorithm and a standard BO approach, IGP-UCB \citep{Chowdhury2017}, which was not originally designed to handle input noise. The theoretical analysis presented in this paper is mainly based on \citeauthor{Chowdhury2017}'s results, which are advantageous in the uncertain-inputs setting due to mild assumptions on the observation noise. However, the results in this section also bring new insights into BO methods for problems with uncertain inputs. We refer the reader to the appendix for complete proofs of the next results. 

\subsection{The uncertain-inputs regret of IGP-UCB}
\label{igp-ucb-bounds}
In the uncertain-inputs setting, IGP-UCB selects target locations $\location_\iterIdx$ by maximising $\gpMean_{\iterIdx-1}(\location)+\beta_\iterIdx\sigma_{\iterIdx-1}(\location)$, where $\gpMean_{\iterIdx-1}$ and $\sigma^2_{\iterIdx-1}$ are respectively the posterior mean and variance of the deterministic-inputs $\gp(0,k)$ given observations $\{(\location_i,\observation_i)\}_{i=1}^{\iterIdx-1}$. For an asymptotic analysis, both the targets $\{\location_\iterIdx\}_{\iterIdx=1}^\infty$ and the equivalent observation noise $\{\obsNoise_\iterIdx\}_{\iterIdx=1}^\infty$, where $\obsNoise_\iterIdx:=\observation_\iterIdx-\expectation[f(\random{\location}_\iterIdx)|\location_\iterIdx]\neq\uObsNoise_\iterIdx$, can be treated as sequences of random variables. At a given round $\iterIdx\geq 1$, the history $\{\location_\primIdx,\obsNoise_\primIdx\}_{\primIdx=1}^\iterIdx$ generates a $\sigma$-algebra $\filtration_\iterIdx$, and the sequence $\{\filtration_\iterIdx\}_{\iterIdx=0}^\infty$ defines a filtration \citep{Bauer1981}. The sub-Gaussian condition on the sequence $\{\obsNoise_\iterIdx\}_{\iterIdx=1}^\infty$ is then formally defined as:
\begin{equation}
\forall \iterIdx \geq 1, ~ \forall \lambda \in \R, \quad \expectation[e^{\lambda \obsNoise_\iterIdx}|\filtration_{\iterIdx-1}] \leq e^{\lambda^2\sigma_\obsNoise^2/2} ~\as~,
\end{equation}
which denotes an upper bound on a conditional expectation \citep{Bauer1981}, so that the inequality above is defined as holding \emph{almost surely} (a.s.).

The results in \citet{Chowdhury2017} bound the cumulative regret of IGP-UCB in terms of the maximum information gain:
\begin{equation}
\mig_{\nIterations} := \max_{\locSet\subset\Sspace:|\locSet|=\nIterations} \mutinfo{\observations_\nIterations,\vec{g}_\nIterations|\locSet}~,
\label{eq:uibo-mig}
\end{equation}
where $\mutinfo{\observations_\nIterations,\vec{g}_\nIterations|\locSet}$ represents the mutual information between $\observations_\nIterations=\vec{g}_\nIterations+\vec{\obsNoise}'_\nIterations$ and $\vec{g}_\nIterations\sim\normal(\vec{0},\mat{K}_\nIterations)$, with $[\mat{K}_\nIterations]_{ij}=k(\location_i,\location_j)$, $\location_i,\location_j\in\locSet$ and $\vec{\obsNoise}'_\nIterations\sim\normal(\vec{0},\lambda\eye)$. Here $\lambda>0$ is the same parameter in \autoref{eq:ugp-mean}. Considering these definitions, we derive the following.

\begin{theorem}[restate=thrboregret,name=IGP-UCB uncertain-inputs regret] \label{thr:bo}
For any $f\in\Hspace_k$, assume that:
\begin{enumerate}
\item the mapping $\location\mapsto\expectation_{\pMeasure_\location^\qry}[f]$ defines a function $g \in \Hspace_k(\Sspace)$ and $\norm{g}_k\leq \bound$;\label{thr:bo-g}
\item $\forall \location \in \Sspace, \Delta f_{\pMeasure_\location^\qry}:=f(\random{\location}^\qry)-\expectation_{\pMeasure^\qry_{\location}}[f]$ is $\sigma_\qry$-sub-Gaussian, for a given $\sigma_\qry > 0$, where $\random{\location}^\qry\sim\pMeasure^\qry_{\location}$;\label{thr:bo-sg}
\item and $\uObsNoise_\iterIdx$ is conditionally $\sigma_\uObsNoise$-sub-Gaussian.\label{thr:bo-n}
\end{enumerate}
Then running IGP-UCB with $\sigma_\obsNoise := \sqrt{\sigma_\qry^2+\sigma_\uObsNoise^2}$ and $\beta_\iterIdx:=\bound+\sigma_\obsNoise\sqrt{2(\mig_{t-1}+1+\log(1/\delta))}$ leads to the same bounds as Theorem 3 in \citet{Chowdhury2017} for the uncertain-inputs cumulative regret of the algorithm. Namely, we have that:
\begin{equation}
\prob{\uRegret_\nIterations \in \set{O}\left(\bound\sqrt{\nIterations\mig_\nIterations}+\sigma_\obsNoise\sqrt{\nIterations(\mig_\nIterations+\log(1/\delta))}\right)} \geq 1 - \delta~.
\end{equation}
\end{theorem}
\begin{proof}[Proof sketch]
Considering Theorem 3 in \citet{Chowdhury2017}, the proof follows almost immediately from the assumptions above. The only detail to notice is that $\obsNoise_\iterIdx:=\observation_\iterIdx-g(\location_\iterIdx) = \uObsNoise_\iterIdx+f(\random{\location}_\iterIdx)-\expectation[f(\random{\location}_\iterIdx)|\location_\iterIdx]=\uObsNoise_\iterIdx+\Delta f_{\pMeasure^\qry_{\location_\iterIdx}}$, which is a $\sigma_\obsNoise$-sub-Gaussian random variable for $\sigma_\obsNoise^2=\sigma_\uObsNoise^2+\sigma_\qry^2$.
\end{proof}

The result above states that, as long as $\sigma_\obsNoise$ is large enough to accommodate for the additional variance in the observations due to noisy-inputs, IGP-UCB maintains bounded regret. Theoretical results bounding the growth of $\mig_{\nIterations}$ are available in the literature. For the squared-exponential kernel on $\R^\locDim$, for example, $\mig_{\nIterations}\in \set{O}((\log\nIterations)^{\locDim+1})$ \citep[Thr. 5]{Srinivas2010}, so that IGP-UCB obtains asymptotically vanishing uncertain-inputs regret in this case.  However, it is possible that the resulting $\sigma_\obsNoise$ makes $\beta_\iterIdx$ impractically large, leading to excessive exploration. The following result addresses these issues.

\begin{proposition}[restate=subgnoise,name={}]\label{thr:subg}
Let $k: \R^d \times \R^d \to \R$ be an at least twice-differentiable positive-definite kernel with finite $\Lipschitz_k^2 \geq \underset{\location \in \R^d}{\sup}\underset{i\in[d]}{\sup}\frac{\partial^2 k(\location,\location')}{\partial\slocation_i\partial\slocation'_i}\big\lvert_{\location=\location'}$. Then, for $\pMeasure \in \Pspace$ and $\random{\location} \sim \pMeasure$, we have that $\Delta f_\pMeasure := f(\random{\location}) - \expectation_{\pMeasure}[f(\random{\location})]$ is $\sigma_F$-sub-Gaussian with:
\begin{enumerate}
\item $\sigma_F=\norm{f}_k \Lipschitz_k \tr(\mat{\Sigma})^{1/2}$, if $\pMeasure$ is Gaussian with covariance matrix $\mat{\Sigma}$;\label{thr:subg-g}
\item  $\sigma_F = \frac{1}{2}\norm{f}_k \Lipschitz_k \sqrt{\sum_{i=1}^{d} \sigma_i^2}$, if $\pMeasure$ has compact support, with $|\tilde{\slocation}_i - \hat{\slocation}_i| \leq \frac{1}{2}\sigma_i$ for each coordinate $i$, where $\vec{\hat{\slocation}}=\expectation_\pMeasure[\random{\location}]$.\label{thr:subg-b}
\end{enumerate}
\end{proposition}
\begin{proof}[Proof sketch]
These results are derived from concentration inequalities available for random variables which are Lipschitz-continuous functions of Gaussian or bounded random variables. For the given kernel, any $f \in \Hspace_k$ is $\norm{f}_k\Lipschitz_k$-Lipschitz continuous.
\end{proof}
\autoref{thr:subg} says that the second condition in \autoref{thr:bo} is met if the execution noise is uniformly bounded or Gaussian. What remains is to verify whether the first assumption in \autoref{thr:bo} can be met.

When working with kernel embeddings of conditional distributions, the assumption that $\location\mapsto\expectation[f(\random{\location})|\location]$ is an element of $\Hspace_k$ is known to be met when the domain $\locDomain$ is discrete, while not necessarily holding for continuous domains \citep{Muandet2016}. As most interesting problems involving uncertain inputs have continuous domains, the following result presents a case where \autoref{thr:bo}'s first assumption holds.

\begin{restatable}{proposition}{thrnoisyf}
\label{thr:noisy-f}
Let $\location\mapsto\pMeasure_{\location}$ be a mapping such that, for any $\location\in\Sspace\subset\locDomain$, $\random{\location} \sim \pMeasure_{\location}\in\Pspace$ is decomposable as $\random{\location} = \location+\locNoise$, where $\locNoise$ is independent and identically distributed, i.e. $\locNoise\sim\pMeasure_\locNoise\in\Pspace$. Assume that $k$ is translation invariant. Then we have that, for any $f\in\Hspace_k$, the mapping $\location\mapsto\expectation_{\pMeasure_\location}[f]$ defines a function $g\in\Hspace_k(\Sspace)$, and $\norm{g}_k \leq \norm{f}_k$.
\end{restatable}
\begin{proof}[Proof sketch]
The proof follows by interpreting $\locNoise$ as a random translation on $f$. Since the kernel is translation invariant, the norm of any $\locNoise$-shifted function $f^\locNoise$ is equivalent to the norm of the original $f$. Then picking $g$ as the restriction of $\expectation_{\pMeasure_\locNoise}[f^\locNoise]\in\Hspace_k$ to $\Sspace\subset\locDomain$ leads to the conclusion.
\end{proof}

\autoref{thr:noisy-f} implies that \autoref{thr:bo} is applicable whenever the execution noise is independent and identically distributed and $k$ is translation-invariant, such as the squared exponential and other popular kernels. However, in cases where the execution noise distribution changes significantly from target to target, algorithms such as uGP-UCB can yield better results.

\subsection{Bounding the regret of uGP-UCB}
\label{sec:ugp-ucb-bounds}
In this section, we analyse the case when uGP-UCB has no access to location estimates $\pMeasure^\loc_\iterIdx$ and uses instead $\Px[\location_\iterIdx]$ with the observations $\dataset_\nIterations=\{\Px[\location_\iterIdx],\observation_{\iterIdx}\}_{\iterIdx=1}^\nIterations$. We will firstly consider an ideal setting, where $\Px = \pMeasure^\qry_\location$, $\forall\location\in\Sspace$, and then a non-ideal scenario. Recall that the regret bounds presented so far depend on the maximum information gain $\mig_\nIterations$. As an analogy, in the case of uGP-UCB, given any $\{\pMeasure_\iterIdx\}_{\iterIdx=1}^\nIterations\subset\Pspace$, we have:
\begin{equation}
\mutinfo{\observations_{\nIterations};\vec{\uncertain{f}}_\nIterations|\{\pMeasure_\iterIdx\}_{\iterIdx=1}^\nIterations} = \frac{1}{2} \log |\eye + \lambda^{-1}\Kp_\nIterations|~,
\label{eq:uibo-ig}
\end{equation}
where $[\Kp_\nIterations]_{ij} = \kp(\pMeasure_i,\pMeasure_j), ~i,j\in\{1,\dots,\nIterations\}$. Let's assume an arbitrary set $\Pspace_s \subset \Pspace$ containing either the query model or the estimated location distributions. As the set $\Pspace_s$ is not necessarily compact, a maximum for $\mutinfo{{\observations}_\nIterations;\vec{\uncertain{f}}_\nIterations|\pSet}$ may not correspond to a given set $\pSet\subset\Pspace_s$. However, we can always define:
\begin{equation}
\uncertain{\gamma}_\nIterations(\Pspace_s) := \sup_{\pSet \subset \Pspace_s: |\pSet| = \nIterations} ~ \mutinfo{{\observations}_\nIterations;\vec{\uncertain{f}}_\nIterations\mid\pSet} ~,
\label{eq:uibo-ig-sup}
\end{equation}
The results presented next will use $\uncertain{\mig}^\qry_\nIterations:=\uncertain{\mig}_\nIterations(\Pspace_\qry)$, where $\Pspace_\qry\subset\Pspace$ is the image of $\Sspace$ under the mapping $\location\mapsto\pMeasure^\qry_\location$. Considering these definitions, the following bounds the uncertain-inputs regret of uGP-UCB.

\begin{theorem}[name=uGP-UCB regret,restate=thrmain]
\label{thr:main}
Let $\delta \in (0,1)$, $f \in \Hspace_k$, and $\bound \geq \lVert f \rVert_{k}$. Consider $\uObsNoise_\iterIdx$ as $\sigma_\uObsNoise$-sub-Gaussian noise. Assume that both $k$ and $\pMeasure_\location^\qry$ satisfy the conditions for $\Delta f_{\pMeasure^\qry_\location}$ to be $\sigma_\qry$-sub-Gaussian, for a given $\sigma_\qry > 0$, for all $\iterIdx \geq 1$. Then, running uGP-UCB with:
\begin{equation}
\beta_\iterIdx = \bound + \sigma_\obsNoise\sqrt{2(\mutinfo{\observations_{\iterIdx-1};\vec{\uncertain{f}}_{\iterIdx-1} | \{\pMeasure^\qry_{\location_\primIdx}\}_{\primIdx=1}^{\iterIdx-1}}+1+\log(1/\delta))}~,
\end{equation}
where $\sigma_\obsNoise:=\sqrt{\sigma_\qry^2 + \sigma_\uObsNoise^2}$, the uncertain-inputs cumulative regret satisfies:
\begin{equation}
\uRegret_\nIterations \in \set{O}\left(\sqrt{\nIterations\uncertain{\mig}_\nIterations^\qry} \left(\bound + \sqrt{\uncertain{\mig}_\nIterations^\qry + \log(1/\delta)}\right) \right)
\label{eq:ugp-ucb-regret-ideal}
\end{equation}
with probability at least $1-\delta$.
\end{theorem}
\begin{proof}[Proof sketch]
This theorem applies the fact that $\kp_\qry(\location,\location'):=\kp(\pMeasure^\qry_\location,\pMeasure^\qry_{\location'})$, for $\location,\location'\in\Sspace$, defines a positive-definite kernel on $\Sspace$ \citep[Lem. 4.3]{Steinwart2008}. By \autoref{thr:exp-f}, we have that $\uncertain{f}\in\Hspace_{\kp}$ and $\uncertain{f}(\pMeasure^\qry_\location)=\expectation[f(\random{\location})|\location]$, for $\random{\location}|\location\sim\pMeasure^\qry_\location$. Then it follows that $g:\location\mapsto\uncertain{f}(\pMeasure^\qry_\location)$ is in $\Hspace_{\kp_\qry}$. As $\uncertain{\mig}^\qry_\nIterations$  is the maximum information gain of a model $\gp(0,\kp_\qry)$, the rest follows from  \autoref{thr:bo}.
\end{proof}

\autoref{thr:main} states that uGP-UCB obtains similar bounds for the uncertain-inputs regret to those of IGP-UCB. However, notice that $\uncertain{\mig}^\qry_\nIterations$, instead of $\mig_{\nIterations}$, appears in \autoref{eq:ugp-ucb-regret-ideal}. The next result shows that $\uncertain{\mig}^\qry_\nIterations\leq\mig_{\nIterations}$, which means smaller regret bounds, in the \iid execution noise case considered previously (\autoref{thr:noisy-f}).

\begin{restatable}{proposition}{thrigiid}
\label{thr:ig-iid}
Consider a compact set $\Sspace\subset\locDomain$, a distribution $\pMeasure_{\locNoise} \in \Pspace$, with $\expectation_{\pMeasure_{\locNoise}}[\locNoise]=0$, and a set:
\begin{equation}
\Pspace_{\locNoise} := \{\pMeasure\in\Pspace\mid \random{\location}=\locMean+\locNoise,\, \locMean\in\Sspace,\locNoise\sim\pMeasure_\locNoise,\,\random{\location}\sim\pMeasure\}\,,
\end{equation}
which is the set of location distributions with mean in $\Sspace$ and affected by \iid $\pMeasure_{\locNoise}$-noise. Assume that $k:\locDomain\times\locDomain\to\R$ is translation invariant, and let $\kp:\Pspace\times\Pspace\to\R$ be defined according to \autoref{eq:ugp-k}. Then we have that:
\begin{equation}
\forall \nIterations \geq 1,\quad \uncertain{\gamma}_\nIterations(\Pspace_\locNoise) \leq \gamma_\nIterations~,
\end{equation}
where $\uncertain{\gamma}_\nIterations$ is defined by \autoref{eq:uibo-ig-sup}, and $\gamma_{\nIterations}$ is the maximum information gain for $\gp(0,k)$.
\end{restatable}
\begin{proof}[Proof sketch]
One can prove that $\mat{K}_\nIterations-\Kp_\nIterations$ is positive definite for a $\mat{K}_\nIterations$ built with $\{\locMean_\iterIdx\}_{\iterIdx=1}^\nIterations\subset\Sspace$. The information gain is a function of the determinant of these matrices, so that the inequality above follows.
\end{proof}

The result above indicates that the uncertain-inputs information gain shrinks as the input noise variance grows. While that might indicate that the optimisation problem becomes easier, if one recalls \autoref{eq:eregret}, the constant $\error_\qry$ grows, making the problem harder.

What remains to verify is the effect of the approximation error between the model $\Px$ and the actual $\pMeasure_{\location}^\qry$. To minimise $\uregret_\iterIdx$, using uGP-UCB with a model $\Px\approx\pMeasure^\qry_\location$ is worth if the approximation error $\hat{\error}_\iterIdx := \max_{\location\in\Sspace}\left|\expectation_{\pMeasure_\location^\qry}[f] - \expectation_{\Px}[f]\right|$ is small. Ideally $\Px$ should be an adaptive model $\Px^\iterIdx$ that can be learnt from past data in  $\dataset_{\iterIdx-1}$ so that $\hat{\error}_\iterIdx\to0$ as $\iterIdx\to\infty$. However, considering execution noise as marginally \iid and Gaussian has been a popular approach when dealing with problems involving uncertain inputs \citep{Mchutchon2011,Nogueira2016}. In this case, we provide an upper bound on $\hat{\error}_\iterIdx$.

\begin{restatable}{proposition}{thrapprox}
\label{thr:approx}
Let $\locDomain=\R^\locDim$, $f \in \Hspace_k$ and $\norm{f}_{k}\leq\bound$. Assume that, for any $\location \in \Sspace\subset\locDomain$, the query distribution $\pMeasure_{\location}^\qry$ is Gaussian with mean $\location$ and positive-definite covariance $\mat{\Sigma}^\qry$. Then, using a Gaussian model $\Px$ with same mean and a given constant positive-definite covariance matrix $\mat{\hat{\Sigma}}$, we have that for any $\location\in\Sspace$:
\begin{equation*}
\left|\expectation_{\pMeasure_\location^\qry}[f] - \expectation_{\Px}[f]\right| \leq \frac{\bound}{2}\sqrt{\tr(\mat{\hat{\Sigma}}^{-1}\mat{\Sigma}^\qry) - \locDim + \log\frac{|\mat{\hat{\Sigma}}|}{|\mat{\Sigma}^\qry|}}\,.
\end{equation*}
\end{restatable}
\begin{proof}[Proof sketch]
This result follows by applying Pinsker's inequality \citep{Boucheron2013} to $\Px$ and $\pMeasure^\qry_\location$.
\end{proof}

\section{Experiments}
\label{sec:exp}
In this section, we present experimental results obtained in simulation with the proposed uGP-UCB algorithm comparing it against other Bayesian optimisation methods: IGP-UCB, with adapted noise model (as in \autoref{thr:bo}), and the unscented expected improvement (UEI) heuristic \citep{Nogueira2016}, which applies the unscented transform to the expected improvement over a conventional GP model. Our aim in this section is to evaluate the performance of these methods in optimisation problems where both the sampling of the objective function and the location at which the sample is taken are significantly noisy. 

\subsection{Objective functions in the same RKHS}
In this experiment, for each trial a different objective function $f\in\Hspace_k$ was generated. The search space was set to the unit square $\Sspace = [0,1]^2 \subset \R^2$. Each $f = \sum_{i=1}^{\nFeatures} \alpha_i k(\cdot,\location_i)$ was generated by uniformly sampling $\alpha_{i}\in [-1,1]$ and support points $\location_i\in\Sspace$, for $i \in\{1,\dots,\nFeatures\}$, with $\nFeatures=30$. Observation noise was set as $\uObsNoise_\iterIdx\sim\normal(0,\sigma_\uObsNoise^2)$ with $\sigma_\uObsNoise=0.1$.

As parameters to verify the theoretical results for the UCB algorithms, we set $\delta = 0.4$, and computed $B = \lVert f \rVert_k$ directly. The querying execution noise in $\pMeasure_\location^\qry$ was i.i.d sampled from $\normal(\vec{0},\sigma_\location^2\mat{I})$ with $\sigma_\location = 0.1$. The output noise parameters for the GP model  were computed according to \autoref{thr:subg}, with each method assuming execution noise coming from $\normal(\vec{0},\hat{\sigma}_\location^2 \mat{I})$. To verify robustness to noise-misspecification, we tested $\hat{\sigma}_\location$ set according to different ratios with respect to the true $\sigma_\location$. Noise on the localisation estimates $\pMeasure^\loc_t$ was set at half the standard deviation of the true execution noise. We directly computed the current information gain $\mutinfo{\observations_{\iterIdx};\vec{\uncertain{f}}_\iterIdx\mid\{\pMeasure^\loc_\primIdx\}_{\primIdx=1}^\iterIdx}$ to set $\beta_\iterIdx$. For both UCB methods and UEI, kernel length-scales were set to 0.1.

\begin{figure}[t]
\centering
\includegraphics[width=\columnwidth]{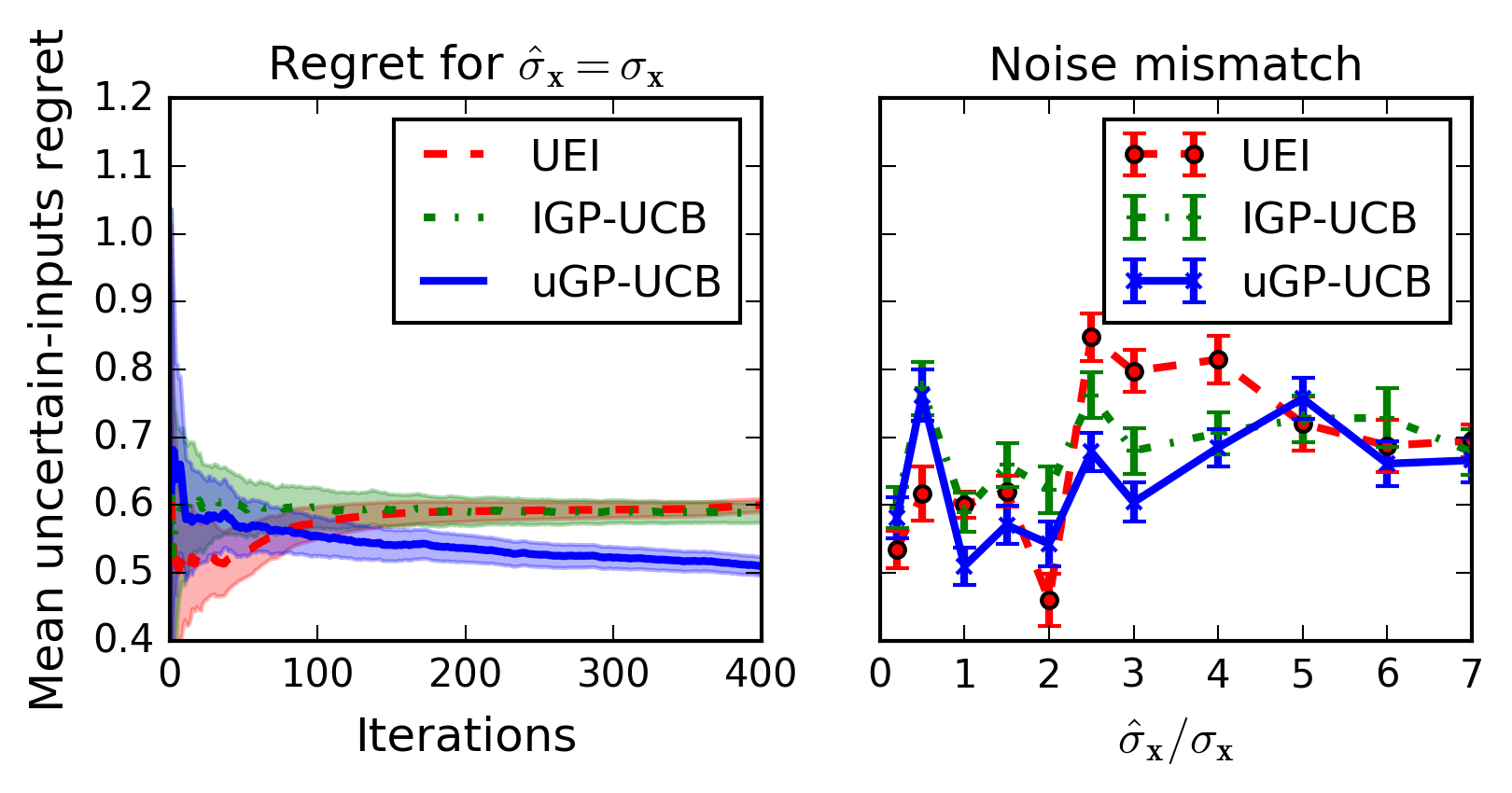}
\caption{Mean uncertain-inputs regret for IGP-UCB, UEI, and uGP-UCB in the optimisation of functions in the same RKHS. On the left, the UCB confidence-bound parameter $\beta_\iterIdx$ was set according to the theoretical results. The plot on the right shows the effect of execution noise model mismatch on each method's regret after running for a total of 400 iterations. Results were averaged over 10 trials, and the shaded areas and error bars correspond to one standard deviation.}
\label{fig:rkhs_functions_perf}
\end{figure}

\paragraph{Results:} \autoref{fig:rkhs_functions_perf} presents performance results, in terms of mean uncertain-inputs regret, i.e. $\uregret_\iterIdx^{\operatorname{avg}}=\frac{1}{\iterIdx}\sum_{\primIdx=1}^\iterIdx\uregret_\primIdx$. This performance metric is an upper bound on the simple regret, since $\min_{\primIdx\leq\iterIdx}\uregret_\primIdx \leq \uregret_\iterIdx^{\operatorname{avg}}$, and allows verifying how close each method gets to the global optimum within $\iterIdx$ iterations. As the plots show, when the execution noise model is correct, with $\hat{\sigma}_\location=\sigma_\location$, uGP-UCB is able to outperform both IGP-UCB and UEI, while every method's performance degrades under mismatch in the execution noise assumption. A larger than needed execution noise variance leads to a large $\beta_\iterIdx$ for the UCB methods, promoting exploration. Querying with a very noisy model $\Px$ also excessively smoothes the GP prior and the acquisition function for uGP-UCB and UEI, respectively. Consequently, each method's model on $f$ tends to a flat function, and none of them is able to make significant improvements after large mismatches, such as $\hat{\sigma}_\location\geq 5\sigma_\location$, as \autoref{fig:rkhs_functions_perf} shows. Despite the loss of performance, uGP-UCB remains as a general lower bound in terms of regret, showing that the proposed method is relatively robust to the effects of mismatch in the execution noise model.

In practice, the convergence rate in \autoref{fig:rkhs_functions_perf} can be improved by setting the UCB parameter $\beta_\iterIdx$ at a fixed low value. As the $\set{O}$ notation indicates, cumulative regret bounds are valid only up to a constant factor. Their main focus is on guaranteeing asymptotic convergence, i.e. $\lim\limits_{\nIterations\to\infty}\uRegret_\nIterations/\nIterations = 0$, as most theoretical results in the UCB literature \citep{Srinivas2010,Chowdhury2017}. To achieve that, the value of the UCB parameter $\beta_\iterIdx$ monotonically increases over iterations, ensuring that the entire search space is explored. The drawback, however, is that excessive exploration decreases performance in the short term. In the next section, we present results where $\beta_\iterIdx$ is fixed.

\subsection{Objective function in different RKHS}
To verify uGP-UCB's performance under incorrect kernel assumptions, the next experiment performed tests with an objective function in a space not matching the GP kernel's RKHS. In particular, we chose the 4-dimensional Michalewicz function, which is a classic benchmark function for global optimisation algorithms \citep{Vanaret2014}, over the domain $\Sspace = [0,\pi]^4$. \autoref{fig:diff-rkhs} presents performance results for fixed $\beta_\iterIdx = 3$. The plots also evaluate each algorithm's sensitivity to the choice of $\beta_\iterIdx$ as a way to asses the robustness of the methods when theoretical assumptions are not met. Input noise was set to $\sigma_\location = 0.1$. As seen, both the proposed uGP-UCB and IGP-UCB can outperform the unscented BO approach UEI. In addition, one can see that uGP-UCB shows consistently better performance than that of IGP-UCB across varying settings for $\beta_\iterIdx$. These results demonstrate that the uGP-UCB algorithm should be able to perform well in situations where its modelling assumptions are not exactly met, such as in scenarios involving physical systems, as presented next.

\begin{figure}[t]
\centering
\includegraphics[width=\columnwidth]{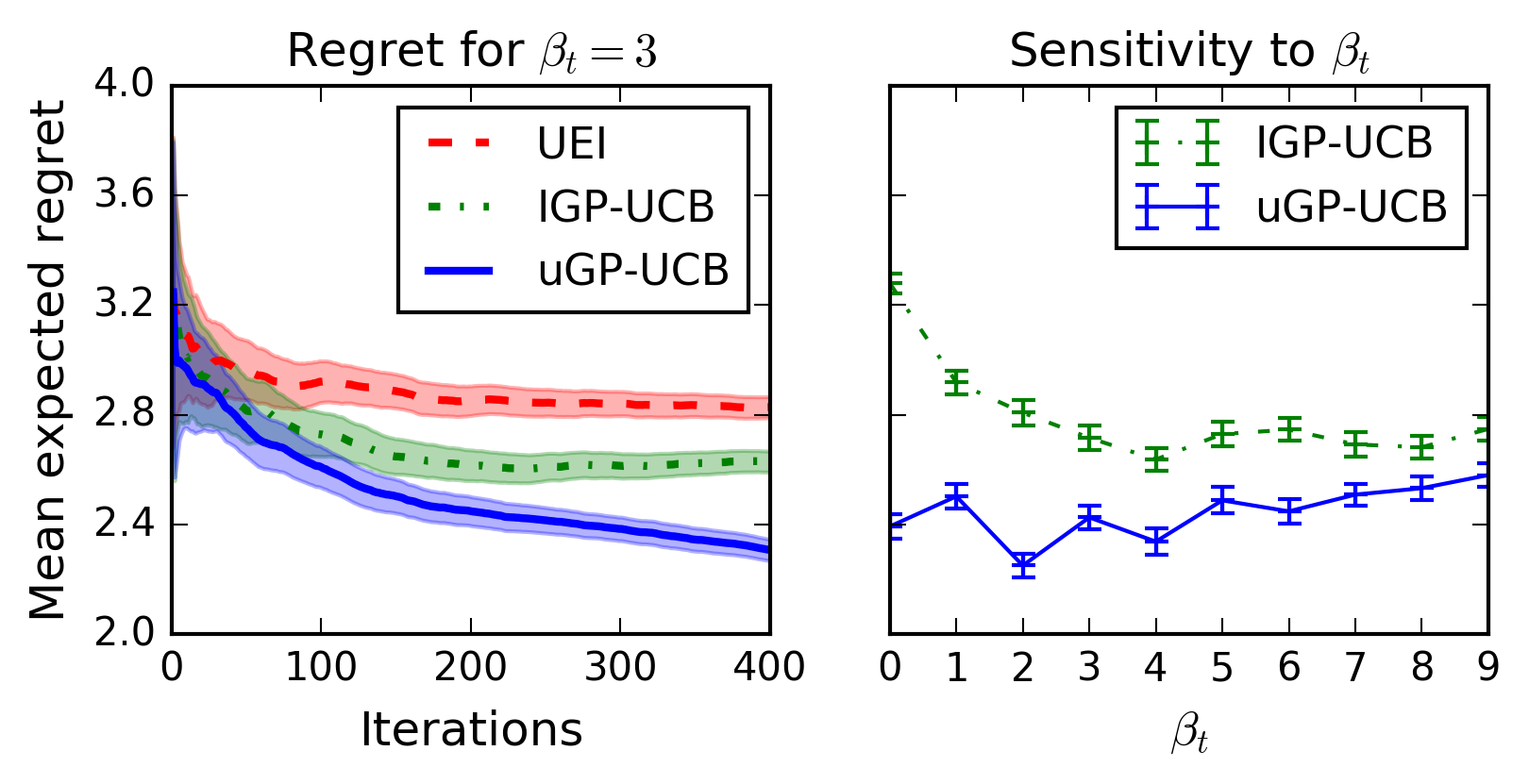}\label{fig:diff-rkhs}
\caption{Optimisation of the Michalewicz function. The plot on the left presents the mean expected regret observed for each algorithm with $\beta_\iterIdx=3$ for UCB methods. On the right, we see how different settings for $\beta_\iterIdx$ affect each UCB method's mean expected regret after 300 iterations. Results were averaged over 10 (left) and 5 (right) trials with shaded areas and error bars corresponding to two standard deviations.}
\end{figure}

\subsection{Robotic exploration problem}
This section presents results obtained in a simulated robotic exploration problem. In this experiment, a robot is set to explore an environmental process. The underlying process is based on the Broom's Barn dataset\footnote{Available at \url{http://www.kriging.com/datasets/}}, consisting of the log-concentration of potassium in the soil of an experimental agricultural area. The robot is allowed to perform up to 30 measurements on different locations. Each BO method sequentially selects the locations where the robot should make a measurement in the usual online decision making process, based on the observations it gets. To simulate the robot, an ATRV platform, we used the OpenRobots' Morse simulator\footnote{Morse: \url{https://www.openrobots.org/morse}}. In this scenario, execution noise is not following a stationary distribution due to the dynamic constraints of the robot, imperfections in motion control, etc. We applied Gaussian noise to the pose information given by the simulator and used pure-pursuit path-following control to guide the robot to the target locations. Location estimates were provided by an extended Kalman filter \citep{Thrun2006}. Hyper-parameters for each GP were learnt online via log-marginal likelihood maximisation. The query noise model for uGP-UCB was set with $\hat{\sigma}_\location^2 = 2$. We set $\beta_t$ at a fixed value, again with $\beta_t = 3$. \autoref{fig:rob} presents the performance of each algorithm in terms of regret. The plots show that uGP-UCB is able to outperform UEI, while performing still better than IGP-UCB in the long run, and with less variance in the outcomes. This result confirms that it is possible to obtain better performance in practical BO problems by taking advantage of distribution estimates and by directly considering execution uncertainty.

\begin{figure}[t]
\centering
\subfloat[Broom's barn data]{\includegraphics[width=0.2\textwidth]{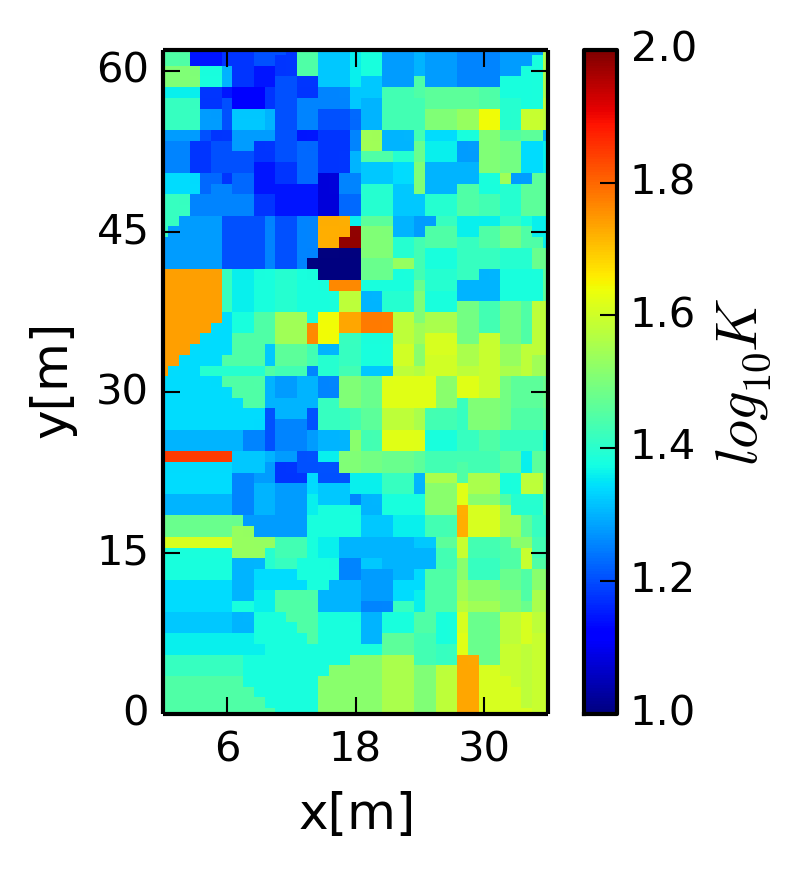}\label{fig:data}}
\subfloat[Robotics problem]{\includegraphics[width=0.27\textwidth]{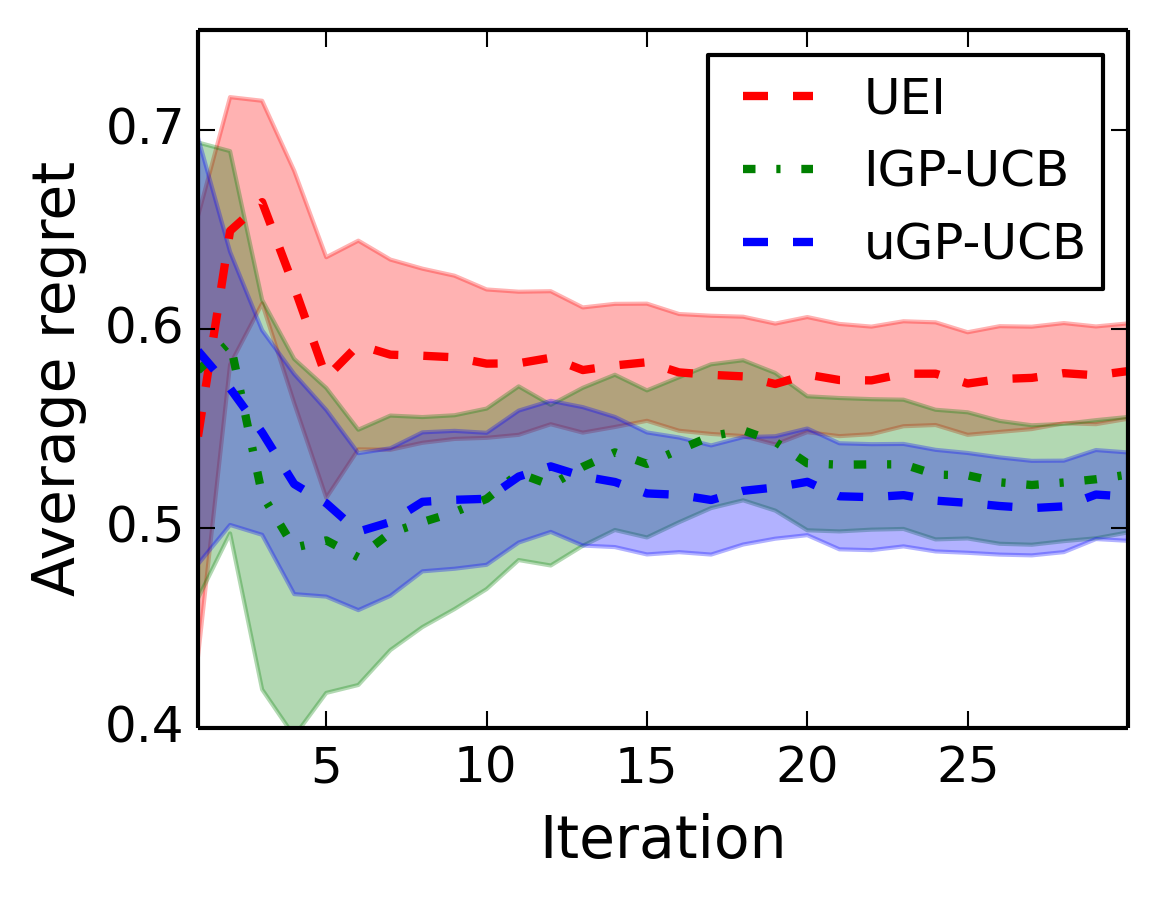}\label{fig:rob}} 
\caption{Robotics exploration experiment: \protect\subref{fig:data} presents the Broom's barn data as distributed over the search space; and \protect\subref{fig:rob} shows the performance of each BO approach, averaged over 4 runs.}
\end{figure}

\section{Conclusion}
\label{sec:conclusion}
In this paper we proposed a novel method to optimise functions where both the sampling of the function as well as the location at which the function is sampled are stochastic. We also provided theoretical guarantees for BO algorithms in noisy-inputs settings. In terms of empirical results, experiments demonstrated that the proposed uGP-UCB shows competitive performance when compared to other BO approaches to input noise. Our method can be applied to many problems where input variates or an agent's state is only partially observable, such as robotics, policy search, stochastic simulations, and others. For future work, it is worth investigating online-learning techniques for the approximate querying distribution $\Px$ that can cope with noisy location estimates and other upper bounds for the uncertain-inputs GP information gain.

\subsection*{Acknowledgements}
We would like to thank the reviewers and Dr. Vitor Guizilini for helpful discussions and the funding agencies CAPES, Brazil, and Data61/CSIRO, Australia.

\setcounter{section}{0}
\renewcommand\thesection{\Alph{section}}
\section{Appendix}

This section presents proofs for auxiliary theoretical results in the main paper. The section starts by presenting some common definitions and lemmas applied by the proofs. More specific background for a given proof, when necessary, will be presented in the section containing the proof itself. Each subsection then presents a proof for each result. In the end, we also present the formulation of the uncertain-inputs squared-exponential kernel (\autoref{sec:k_use}) used in experiments. For reference, a notation summary is presented in \autoref{tab:notation}.

The main theorems in this paper are based on the following result by \citet{Chowdhury2017}, restated here for convenience.

\begin{theorem}[{\citet[Theorem 3]{Chowdhury2017}}]
\label{thr:sbo}
Let $\delta \in (0,1)$, $\lVert f \rVert_{k} \leq \bound$, and $\obsNoise_\iterIdx$ be conditionally $\sigma_\obsNoise$-sub-Gaussian noise. Then, running IGP-UCB with $\beta_\iterIdx = \bound + \sigma_\obsNoise\sqrt{2({\gamma}_{\iterIdx-1}+1+\log(1/\delta))}$ for $f \in \Hspace_{k}(\Sspace)$, and a compact $\Sspace\subset\R^\locDim$, the cumulative regret of the algorithm is bounded by $\set{O}(\sqrt{\nIterations}(\bound\sqrt{{\gamma}_\nIterations}+{\gamma}_\nIterations))$ with high probability. Specifically, we have that:
\begin{equation}
\prob{\Regret_\nIterations \in \set{O}\left(\bound\sqrt{\nIterations\gamma_\nIterations}+\sqrt{\nIterations(\gamma_\nIterations+\log(1/\delta))}\right)} \geq 1 - \delta~.
\label{eq:sbo-bound}
\end{equation}
\end{theorem}

The following are common definitions and known theoretical results applied by different proofs.

\begin{definition}\label{def:subg}
For a given $\sigma_\rv>0$, a real-valued random variable $\rv$ is said to be $\sigma_\rv$-sub-Gaussian if:
\begin{equation}
\forall \lambda \in \R, ~ \expectation[e^{\lambda \rv}] \leq e^{\lambda^2\sigma_\rv^2/2} ~.
\end{equation}
\end{definition}

\begin{definition}[Bounded linear operator]
\label{def:blop}
A linear operator $\operator{L}:\anyspace\to\anotherspace$ mapping a vector space $\anyspace$ to a vector space $\anotherspace$, both over the same field, is any operator such that, for all $\anyelement,\anyelement'\in\anyspace$ and any scalar $\alpha$:
\begin{axioms}
\item $\operator{L}(\anyelement+\anyelement') = \operator{L}\anyelement + \operator{L}\anyelement'$
\item $\operator{L}(\alpha\anyelement) = \alpha\operator{L}\anyelement$
\end{axioms}
If $\anyspace$ and $\anotherspace$ are normed vector spaces, the operator $\operator{L}$ is bounded if there is a constant $\anyconstant\in\R$ such that:
\begin{equation}
\forall\anyelement\in\anyspace,\quad \norm{\operator{L}\anyelement}_{\anotherspace} \leq \anyconstant \norm{\anyelement}_{\anyspace}~.
\end{equation}
The smallest $c$ satisfying the above is called the norm of the operator $\operator{L}$, denoted by $\norm{\operator{L}}$.
\end{definition}

\begin{lemma}[{Bounded linear extension theorem \citep[Thr. 2.7-11]{Kreyszig1978}}]
\label{thr:ble}
Let $\operator{M}:\anydomain\to\anotherspace$ be a bounded linear operator, where $\anydomain$ lies in a normed vector space $\anyspace$, and $\anotherspace$ is a Banach space. Then $\operator{M}$ has an extension $\operator{\overline{M}}:\overline{\anydomain}\to\anotherspace$, where $\operator{\overline{M}}$ is a bounded linear operator with norm $\norm{\operator{\overline{M}}}=\norm{\operator{M}}$, and $\overline{\anydomain}$ denotes the closure of $\anydomain$ in $\anyspace$.
\end{lemma}

\begin{table}[t]
\caption{Notation}
\begin{tabular}[ht]{ll}
\hline
$\R$ & the field of real numbers, or the real line \\
$\R^\locDim$ & the Euclidean vector space of dimension $\locDim$\\
$\locDomain$ & domain of BO's objective function\\
$\Sspace$ & BO's search space, a subset of $\locDomain$\\
$\Pspace$ & set of all probability measures on $\locDomain$\\
$\location$ & a location vector, $\location\in\R^\locDim$\\
$\random{\location}$ & an $\R^\locDim$-valued random variable\\
$f$ & deterministic-inputs function\\
$\uncertain{f}$ & uncertain-inputs function, i.e. $\uncertain{f}:\Pspace\to\R$\\
$\pMeasure$ & a probability measure or distribution\\
$\pMeasure^\loc_t$ & location distribution informed after query\\
$\pMeasure^\qry_\location$ & query location distribution given target $\location$\\
$\hat{\pMeasure}_\location$ & model for $\pMeasure^\qry_{\location}$ used by uGP-UCB\\
$\meanMap_\pMeasure$ & kernel mean embedding of $\pMeasure$\\
$k$ & a positive-definite kernel\\
$\Hspace_{k}$ & the RKHS of $k$\\
$\Hspace_{k}(\Sspace)$ & restriction of $\Hspace_{k}$ to a subdomain $\Sspace$\\
$\Hspace_{k}^0$ & the pre-Hilbert space defined by $k$\\
$\overline{\Hspace_{k}^0}$ & the closure of the pre-RKHS $\Hspace_{k}^0$ in $\Hspace_{k}$\\
\hline
\end{tabular}
\label{tab:notation}
\end{table}

\subsection{Proof of \autoref{thr:exp-f}}
\label{sec:exp-f-proof}
\expectedfunction*
\begin{proof}
\autoref{thr:exp-f} basically follows from the presence of Dirac measures in $\Pspace$, which allow transforming point evaluations into expectations. For the proof, we will first derive a bounded linear operator $\map:\Hspace_{k}\to\Hspace_{\kp}$ satisfying the conditions in \autoref{eq:exp-f-def}. From \autoref{def:blop}, it is not hard to see that any bounded linear operator is also continuous \citep[see][Thr. 2.7-9]{Kreyszig1978}. The isometric relationship between $\Hspace_{k}$ and $\Hspace_{\kp}$ depends on the existence of a bijective isometry between the two Hilbert spaces. We will prove that by showing that $\map$, which is an isometry, has an inverse $\map^{-1}:\Hspace_{\kp}\to\Hspace_{k}$.

To facilitate the analysis, we start by working with the pre-RKHS associated with $k$, which is defined as:
\begin{equation}
\Hspace_k^0 := \lspan\{k(\cdot,\location)\mid\location\in\locDomain\}~,
\label{eq:pre-rkhs}
\end{equation}
where $\lspan$ denotes the linear span, i.e. $\Hspace_k^0$ is the set of all linear combinations of the vectors $k(\cdot,\location)$, $\location\in\locDomain$. Since $\Hspace_k^0$ is dense in $\Hspace_k$ \citep[Thr. 4.21]{Steinwart2008}, any bounded linear map defined on $\Hspace_k^0$ can be extended to the full $\Hspace_k$ by \autoref{thr:ble}. 

Given any $f = \sum_{\primIdx=1}^{\nFeatures} \alpha_\primIdx k(\cdot,\location_\primIdx) \in\Hspace_k^0$, define the map $\map_0: \Hspace_{k}^0 \to \Hspace_{\kp}$ by:
\begin{equation}
\map_0 f = \sum_{\primIdx=1}^{\nFeatures} \alpha_\primIdx \kp(\cdot,\Dirac_{\location_\primIdx}) \in \Hspace_{\kp}~,
\end{equation}
where $\Dirac_\location \in \Pspace$ is the Dirac measure centred on $\location$. From the definition of $\meanMap$ in \autoref{eq:kme}, note that $k(\cdot,\location)=\meanMap_{\Dirac_\location}$ for any $\location\in\locDomain$. With this property and the definition of $\kp$ (\autoref{eq:ugp-k}), for any $f\in\Hspace_k^0$, we have that:
\begin{equation}
\begin{split}
\forall \pMeasure\in\Pspace, \quad \map_0 f(\pMeasure)
&= \sum_{\primIdx=1}^{\nFeatures} \alpha_\primIdx \kp\left(\pMeasure,\Dirac_{\location_\primIdx}\right)\\
&= \sum_{\primIdx=1}^{\nFeatures} \alpha_\primIdx \inner{\meanMap_{\pMeasure},k(\cdot,\location)}_k\\
&= \inner{f,\meanMap_\pMeasure}_k = \expectation_\pMeasure[f]
\end{split}
\end{equation}
Linearity follows, since, for any $f,g\in\Hspace_k^0$:
\begin{equation}
\begin{split}
\map_0(f+g)(\pMeasure) &= \expectation_\pMeasure[f+g]\\
&= \expectation_\pMeasure[f]+\expectation_\pMeasure[g]\\
&= \map_0 f(\pMeasure) + \map_0 g(\pMeasure)\,, \, \forall \pMeasure\in\Pspace\,,
\end{split}
\end{equation}
and, for any $\alpha\in\R$:
\begin{equation}
\begin{split}
\map_0(\alpha f)(\pMeasure) &= \expectation_\pMeasure[\alpha f]\\ 
&= \alpha \expectation_\pMeasure[f]\\
&= \alpha\map_0 f(\pMeasure)\,, \,\forall \pMeasure\in\Pspace~.
\end{split}
\end{equation}
Furthermore, for any $f := \sum_{\primIdx=1}^{\nFeatures} \alpha_\primIdx k(\cdot,\location_\primIdx) \in\Hspace_k^0$, the RKHS norm of $\uncertain{f}=\map_0 f$ is such that:
\begin{equation}
\begin{split}
\norm{\uncertain{f}}_{\kp}^2 &= \sum_{\primIdx=1}^{\nFeatures} \sum_{\secIdx=1}^{\nFeatures} \alpha_\primIdx \alpha_\secIdx \kp\left(\Dirac_{\location_\primIdx},\Dirac_{\location_\secIdx}\right)\\
&= \sum_{\primIdx=1}^{\nFeatures} \sum_{\secIdx=1}^{\nFeatures} \alpha_\primIdx \alpha_\secIdx k(\location_\primIdx,\location_\secIdx) = \norm{f}_k^2~.
\end{split}
\end{equation}
Therefore, $\map_0$ represents a bounded linear operator.  Applying \autoref{thr:ble} to $\map_0$ yields the first statement in \autoref{thr:exp-f}. For the remaining steps, let $\map := \overline{\map_0}$.

For $\Hspace_{\kp}$ to be isometric to $\Hspace_k$, the mapping by $\map$ needs to be invertible. As a bounded linear operator between Hilbert spaces, $\map$ has a unique adjoint $\map^*:\Hspace_{\kp}\to\Hspace_{k}$ with $\norm{\map^*}=\norm{\map}$ \citep[Thm. 3.9-2]{Kreyszig1978}. In our case, $\map^*$ is such that, given any $\uncertain{f}\in\Hspace_{\kp}$:
\begin{equation}
\begin{split}
\uncertain{f}(\pMeasure) &= \inner{\uncertain{f},\kp(\cdot,\pMeasure)}_{\kp}\\
&= \inner{\uncertain{f},\map \meanMap_\pMeasure}_{\kp}\\
&= \inner{\map^* \uncertain{f},\meanMap_\pMeasure}_k\\
&= \expectation_\pMeasure[\map^* \uncertain{f}]\,,\quad\forall \pMeasure\in\Pspace\,.
\end{split}
\end{equation}
Setting $\uncertain{f}:=\map f$, for $f \in\Hspace_k$, in the equation above, we see that $\expectation_\pMeasure[f]=\map f(\pMeasure)=\expectation[\map^*\map f]$, so that $\map^*=\map^{-1}$, which concludes the proof.
\end{proof}

\subsection{Proof of \autoref{thr:bo}}
\label{sec:proof-bo}
\thrboregret*
\begin{proof}
\autoref{thr:bo} establishes sufficient conditions for \autoref{thr:sbo} to be applicable to the noisy-inputs settings. The observation noise, as perceived by the GP model, is $\obsNoise_\iterIdx := \observation_{\iterIdx} - g(\location_\iterIdx)$, where $g$ follows the definition in \autoref{thr:bo-g} and $\location_{\iterIdx}$ is the location selected by IGP-UCB according to the setting for $\beta_\iterIdx$ in \autoref{thr:bo}. Observations $\observation_{\iterIdx}$ are taken at $\random{\location}^\qry_\iterIdx \sim \pMeasure_{\location_\iterIdx}^\qry$, instead, yielding:
\begin{equation}
\obsNoise_\iterIdx = \observation_{\iterIdx} - g(\location_\iterIdx) = \uObsNoise_\iterIdx + f(\random{\location}^\qry_\iterIdx)-\expectation_{\pMeasure^\qry_{\location_\iterIdx}}[f] = \uObsNoise_\iterIdx + \Delta f_{\pMeasure_{\location_\iterIdx}^\qry}~.
\end{equation}
Given that $\location_\iterIdx$ is $\filtration_{\iterIdx-1}$-measurable, as it is predictable given $\{\location_{i},\obsNoise_i\}_{i=1}^{\iterIdx-1}$, we have that $\Delta f_{\pMeasure_{\location_\iterIdx}^\qry}$ is $\sigma_F$-sub-Gaussian when conditioned on $\filtration_{\iterIdx-1}$. By assumption \ref{thr:bo-n}, $\uObsNoise_\iterIdx$ is conditionally sub-Gaussian. Since $\uObsNoise_\iterIdx$ and $\Delta f_{\pMeasure_{\location_\iterIdx}^\qry}$ are independent given $\filtration_{\iterIdx-1}$, we have that:
\begin{equation}
\begin{split}
\expectation[\exp{(\lambda \obsNoise_\iterIdx)}|\filtration_{\iterIdx-1}] &= \expectation\left[\exp{\left(\lambda\left( \uObsNoise_\iterIdx+\Delta f_{\pMeasure_{\location_\iterIdx}^\qry}\right)\right)}\middle|\filtration_{\iterIdx-1}\right]\\
&=\expectation\left[\exp{\left(\lambda\uObsNoise_\iterIdx\right)}\exp{\left(\lambda\Delta f_{\pMeasure_{\location_\iterIdx}^\qry}\right)}\middle|\filtration_{\iterIdx-1}\right]\\
 &\leq e^{\lambda^2\sigma_\uObsNoise^2/2} e^{\lambda^2\sigma_F^2/2}\\
 &= e^{\lambda^2(\sigma_\uObsNoise^2+\sigma_F^2)/2}\\
 &= e^{\lambda^2\sigma_\obsNoise^2/2} ~\as~,\quad \forall \lambda\in\R\,,
\end{split}
\end{equation}
so that $\obsNoise_\iterIdx$ is conditionally $\sigma_\obsNoise$-sub-Gaussian. 

Assumption \ref{thr:bo-g} states that $g \in \Hspace_k(\Sspace)$, meeting the remaining requirement for \autoref{thr:sbo}. Therefore, running IGP-UCB with $\sigma_\obsNoise$ and $\bound \geq \norm{g}_k$, following the settings in \autoref{thr:sbo}, leads to cumulative regret bounds for $g$ as in \autoref{eq:sbo-bound}. From the definition in \autoref{eq:uregret}, the cumulative regret $\Regret_\nIterations$ for $g$ is equivalent to the uncertain-inputs cumulative regret $\uRegret_\nIterations$ for $f$, which leads to the conclusion in \autoref{thr:bo}.
\end{proof}

\subsection{Proof of \autoref{thr:subg}}
\label{sec:proof-subg}
To prove \autoref{thr:subg}, we will make use of the following theoretical background.

\begin{definition}[Bounded differences property]\label{def:b-diff}
Let $\location = [\slocation_1, \dots, \slocation_d]^\transpose$ and: \begin{equation}
\location'_i = [\slocation_1, \dots, \slocation_{i-1}, \slocation'_i, \slocation_{i+1}, \dots, \slocation_d]^\transpose~,
\end{equation}
where $\slocation_i, \slocation'_i \in\locDomain_i \subset \R$ and $\locDomain = (\locDomain_1 \times \dots \times \locDomain_d)$. A function $f:\locDomain\to \R$ has the bounded differences property if:
\begin{equation}
|f(\location) - f(\location'_i)| \leq c_i, ~ \forall i \in \{1,\dots,d\} ~,
\end{equation}
where $c_i$ are non-negative constants.
\end{definition}

\begin{lemma}[Corollary 4.36 in \citet{Steinwart2008}]\label{thr:f-lip}
Let $f \in \Hspace_k$, where $k: \locDomain \times \locDomain \to \R$ is a twice-differentiable kernel on $\locDomain \subseteq \R^\locDim$. Then $f$ has bounded first-order partial derivatives, such that for any $\location \in \locDomain$:
\begin{equation}
\left|\frac{\partial f(\location)}{\partial \slocation_i} \right| \leq \lVert f \rVert_k \sqrt{\frac{\partial^2 k(\location,\location')}{\partial\slocation_i\partial\slocation'_i}\Bigm|_{\location'=\location}} ~.
\end{equation}
\end{lemma}

\begin{lemma}[Theorem 5.5 in \citet{Boucheron2013}]\label{thr:tis}
Let $\random{\location} \sim \normal(\vec{0},\mat{I})$ be an $\R^\locDim$-valued standard Gaussian random vector. Let $f:\R^\locDim\to\R$ denote a $\Lipschitz$-Lipschitz function, i.e.:
\begin{equation}
|f(\location)-f(\location')| \leq \Lipschitz \lVert \location - \location' \rVert_2, ~\forall \location, \location' \in \R^\locDim ~.
\end{equation}
Then, for all $\lambda \in \R$:
\begin{equation}
\expectation[e^{\lambda (f(\random{\location}) - \expectation[f(\random{\location})])}] \leq e^{\frac{1}{2}\lambda^2 \Lipschitz^2} ~. \label{eq:tis}
\end{equation}
\end{lemma}

Now we can proceed to the proof of \autoref{thr:subg}, which is restated below.
\subgnoise*
\begin{proof}
The following proof is split in two parts. The derivation firstly covers the case where the inputs follow a Gaussian distribution and then the case for arbitrary probability distributions with compact support.

\paragraph{(\ref{thr:subg-g}) Gaussian inputs:}
In the case of Gaussian inputs, \autoref{thr:subg} is a direct consequence of \autoref{thr:tis} when applied to functions $f \in \Hspace_k$. Notice that, by the definition of $\Hspace_k$, any $f$ in it is continuously differentiable and Lipschitz continuous according to \autoref{thr:f-lip}. All we have to do is to generalise the inequality in \autoref{eq:tis} for the case of general Gaussian random vectors $\random{\location} \sim \normal(\locMean,\bm{\Sigma})$.

If $\random{\location}^s$ is a standard Gaussian random vector, $\random{\location} = \locMean + \mat{A}\random{\location}^s$, where $\bm{\Sigma} = \mat{A}\mat{A}^\transpose$, due to the translational and rotational invariance of Gaussian random vectors. We can define a function $g$, such that:
\begin{equation}
g(\random{\location}^s) = f(\locMean + \mat{A}\random{\location}^s) = f(\random{\location})~. \label{eq:g}
\end{equation}
Since $f$ is Lipschitz continuous, $g$ also is, for some Lipschitz constant $\Lipschitz_g$. Then we can apply \autoref{thr:tis} to $g$, which yields:
\begin{equation}
\expectation\left[e^{\lambda (g(\random{\location}^s) - \expectation[g(\random{\location}^s)])}\right] \leq e^{\frac{1}{2}\lambda^2 \Lipschitz_g^2} ~. \label{eq:g-subg}
\end{equation}
In addition, by definition (\autoref{eq:g}), $g(\random{\location}^s)$ and $f(\random{\location})$ follow the same distribution, so that $\expectation[f(\random{\location})] = \expectation[g(\random{\location}^s)]$. As a result, $\Delta f_\pMeasure = f(\random{\location}) - \expectation[f(\random{\location})]$ is $\Lipschitz_g$-sub-Gaussian, according to \autoref{def:subg}.

Now $\Lipschitz_g$ is any constant uniformly upper-bounding the Euclidean norm of $g$'s gradient, and: 
\begin{equation}
\begin{split}
\rVert \nabla g \lVert_2^2 &= \lVert \mat{A}^\transpose \nabla f \rVert_2^2\\
&= \nabla f^\transpose \mat{A} \mat{A}^\transpose \nabla f\\
&= \nabla f^\transpose \bm{\Sigma} \nabla f\\
\end{split}
\end{equation}
Without loss of generality, let's assume that $\bm{\Sigma}$ is a matrix of diagonal entries $\sigma_i^2, 1\leq i \leq d$. Then we have that:
\begin{equation}
\norm{\nabla g}_2^2= \sum_{i=1}^{d} \sigma_i^2 \left|\frac{\partial f}{\partial \slocation_i} \right|^2 \leq \Lipschitz_f^2 \mathrm{tr}(\bm{\Sigma}) ~,
\label{eq:g-grad-bound}
\end{equation}
where $\Lipschitz_f = \lVert f \rVert_k \Lipschitz_k$. Therefore, the inequality in \autoref{eq:g-subg} holds for $\Lipschitz_g = \Lipschitz_f \sqrt{(\mathrm{tr}(\bm{\Sigma}))}$.

For a non-diagonal $\bm{\Sigma}$, by spectral decomposition, we have that $\bm{\Sigma} = \mat{V}\bm{\Lambda}\mat{V}^\transpose$, where $\bm{\Lambda}$ is a diagonal matrix composed of $\bm{\Sigma}$'s eigenvalues and $\mat{V}\mat{V}^\transpose = \mat{I}$. Observe that the result in \autoref{eq:g-grad-bound} would also hold for a zero-mean Gaussian random vector $\random{\location}^v$ with covariance matrix $\bm{\Lambda}$. Then we could define $h(\random{\location}^v) = f(\locMean + \mat{V}\random{\location}^v)$ and follow similar steps to the ones we took for $g$. However, $f$ and $h$, as defined, have the same Lipschitz constant, since:
\begin{equation}
\lVert \mat{V}\location - \mat{V}\location' \rVert_2^2 = (\location - \location')^\transpose \mat{V}^\transpose \mat{V} (\location - \location') = \rVert \location - \location' \lVert_2^2 ~,
\end{equation}
where we applied $\mat{V}^\transpose \mat{V} = \mat{V}\mat{V}^\transpose = \mat{I}$. In addition, as $\bm{\Sigma}$ is positive definite, $\mathrm{tr}(\bm{\Sigma}) = \mathrm{tr}(\bm{\Lambda})$. Therefore, the same result in \autoref{eq:g-grad-bound} holds for general $\bm{\Sigma}$ and $\locMean$, which can also be seen as a consequence of the translational and rotational invariance of Gaussian random vectors. Making $\sigma_F = \Lipschitz_g = \rVert f \lVert_k \Lipschitz_k \mathrm{tr}(\bm{\Sigma})^{1/2}$ concludes the first part of the proof.

\paragraph{(\ref{thr:subg-b}) Distributions with compact support:} By \autoref{thr:f-lip}, we can observe that $f \in \Hspace_k(\locDomain)$ is Lipschitz continuous with respect to the 1-norm on $\R^\locDim$, in particular:
\begin{equation}
|f(\location) - f(\location')| \leq \lVert f \rVert_k \Lipschitz_k \lVert \location - \location' \rVert_1, ~\forall \location,\location' \in \R^\locDim ~,
\end{equation}
where $\Lipschitz_k \geq 0$ is any constant such that $\Lipschitz_k^2 \geq \underset{\location \in \locDomain}{\sup}\underset{i\in[d]}{\sup}\frac{\partial^2 k(\location,\location')}{\partial\slocation_i\partial\slocation'_i}|_{\location=\location'}$. Therefore, according to \autoref{def:b-diff}, $f$ satisfies the bounded differences property for any $\location$ in the support of $\pMeasure$ with $c_i = \lVert f \rVert_k \Lipschitz_k \sigma_i$. Applying McDiarmid's inequality \citep{McDiarmid1989}, we have that:
\begin{equation}
\prob{|f(\random{\location})- \expectation_\pMeasure(f)| \geq t} \leq 2\exp\left(-\frac{2t^2}{\lVert f \rVert_k^2 \Lipschitz_k^2 \sum_{i=1}^{d} \sigma_i^2}\right)~.
\end{equation}
As a result, $\Delta f_\pMeasure$ is $\sigma_F$-sub-Gaussian with $\sigma_F = \frac{1}{2}\lVert f \rVert_k \Lipschitz_k \sqrt{\sum_{i=1}^{d} \sigma_i^2}$, according to \autoref{def:subg} and Lemma 2.2 in \citet{Boucheron2013}. This concludes the proof.
\end{proof}

\subsection{Proof of \autoref{thr:noisy-f}}
\label{sec:proof-noisy-f}
\thrnoisyf*
\begin{proof}
To prove this result, we will consider properties of the inner product in $\Hspace_k$ when $k$ is translation invariant. These properties essentially allow us to transfer the noise in the evaluation of $f$ to $f$ itself and then represent $g$ as the expectation of this noisy version of $f$. Similar to the proof of \autoref{thr:exp-f}, we start by defining an operator on $\Hspace_k^0$ (see \autoref{eq:pre-rkhs}) and then extend it to $\Hspace_k$ by \autoref{thr:ble}.

To develop the proof, we need to represent $f$ in terms of the kernel $k$. Let $f = \sum_{i=1}^{\nFeatures} \alpha_{i} k(\cdot,\location_{i}) \in \Hspace_k^0$, which is the pre-Hilbert space of $k$. Considering the evaluation of the expected value of $f$, we have that:
\begin{equation}
\begin{split}
\forall\location\in\Sspace, ~\expectation_{\pMeasure_\location}[f] &= \expectation_{\locNoise\sim\pMeasure_\qry}[f(\location+\locNoise)]\\
&= \expectation_{\locNoise\sim\pMeasure_\qry}\left[\sum_{i=1}^{\nFeatures} \alpha_{i} k(\location+\locNoise,\location_{i})\right]  ~.
\end{split}
\label{eq:noisy-f-exp}
\end{equation}
For a fixed $\locNoise\in\R^\locDim$, we have that $k(\location+\locNoise, \location') = k(\location,\location'-\locNoise), ~\forall \location,\location'\in\locDomain$, by translation invariance. Applying this property, we obtain:
\begin{equation}
\begin{split}
f(\location+\locNoise) &= \sum_{i=1}^{\nFeatures} \alpha_{i} k(\location+\locNoise,\location_{i})\\
&= \sum_{i=1}^{\nFeatures} \alpha_{i} k(\location,\location_{i}-\locNoise)\\
&= \inner{\sum_{i=1}^{\nFeatures} \alpha_{i} k(\cdot,\location_{i}-\locNoise)~,~k(\cdot,\location)}_k\\
&= f^\locNoise(\location)\\
\end{split}
\end{equation}
where $f^\locNoise := \sum_{i=1}^{\nFeatures} \alpha_{i} k(\cdot,\location_{i}-\locNoise)$ is equivalent to a version of $f$ with inputs shifted by $\locNoise$. As the shift $\locNoise$ is the same for all $\location_i, ~i\in\{1,\dots,\nFeatures\}$, the norm is unaffected:
\begin{equation}
\begin{split}
\norm{f^\locNoise}_k^2 = \inner{f^\locNoise,f^\locNoise}_k &= \sum_{i=1}^{\nFeatures} \sum_{j=1}^{\nFeatures} \alpha_{i} \alpha_{j} k(\location_{i}-\locNoise,\location_{j}-\locNoise)\\
&= \sum_{i=1}^{\nFeatures} \sum_{j=1}^{\nFeatures} \alpha_{i} \alpha_{j} k(\location_{i},\location_{j})\\
&= \inner{f,f}_k = \norm{f}_k^2~,
\label{eq:proof-noisy-f-norms}
\end{split}
\end{equation}
where the second equality follows by translational invariance. Defining the mapping $f\mapsto f^\locNoise$ as an operator from $\Hspace_k^0$ to $\Hspace_k$, one can easily show that this operator is linear and bounded. Applying \autoref{thr:ble}, then we have that $f\mapsto f^\locNoise$ is actually well defined over the entire $\overline{\Hspace_k^0}=\Hspace_k$.

Now we can return to the derivation in \autoref{eq:noisy-f-exp}. Since $k$ is measurable, we have that $\locNoise\mapsto f^\locNoise$ defines a $\Hspace_k$-valued random variable \citep[Ch. 4]{Berlinet2004}. In addition, as $\norm{f}_k$ is finite, $\locNoise\mapsto\norm{f^\locNoise}_k$ is bounded, so that expectations are well defined as Bochner integrals \citep[see][Ch. 4, Sec. 5]{Berlinet2004}. Applying these results to \autoref{eq:noisy-f-exp} yields:
\begin{equation}
\begin{split}
\forall f \in \Hspace_k, \forall\location\in\Sspace, ~\expectation_{\pMeasure_\location}[f] &= \expectation_{\locNoise\sim\pMeasure_\qry}[f^\locNoise(\location)]\\ 
&=\inner{\expectation_{\locNoise\sim\pMeasure_\qry}[f^\locNoise],k(\cdot,\location)}_k~.
\end{split}
\end{equation}
Defining $g':= \expectation_{\locNoise\sim\pMeasure_\qry}[f^\locNoise]$ and restricting the domain to $\Sspace$, set $g:=g'\vert_\Sspace\in\Hspace_k(\Sspace)$. By the boundedness of the Bochner integral \citep[see][Ch. 2]{Mandrekar2015}, which defines $\expectation[f^\locNoise]$, we know that: 
\begin{equation}
\norm{g'}_k = \norm{\expectation[f^\locNoise]}_k \leq \expectation[\norm{f^\locNoise}_k] = \norm{f}_k~.
\end{equation}
Regarding the norm of the domain-restricted function, we then have that \citep{Aronszajn1950}:
\begin{equation}
\norm{g}_{\Hspace_k(\Sspace)} = \inf_{h\in\Hspace_k:h\vert_\Sspace=g}\norm{h}_k \leq \norm{g'}_k \leq \norm{f}_k~.
\end{equation}
The result in \autoref{thr:noisy-f} immediately follows, which concludes the proof.
\end{proof}

\subsection{Proof of \autoref{thr:main}}
The proof for the main result concerning uGP-UCB will make use of the following background.

\begin{lemma}[name={Theorem 2.9 in \citet{Saitoh2016}}]
\label{thr:rkhs-composition}
Consider a kernel $k:\anydomain\times\anydomain\to\R$ and an arbitrary mapping $\anyfunction:\anyspace\to\anydomain$. Set
\begin{equation}
\set{Z}_\anyfunction := \bigcap_{\anyelement\in\anyspace} \nullspace\left(\operator{E}_{\anyfunction(\anyelement)}\right) \subset \Hspace_{k}~,
\end{equation}
where, given $w\in\anydomain$, $\operator{E}_w$ denotes the evaluation functional and $\nullspace(\operator{E}_w)$ denotes the null space of $\operator{E}_w$. Let $\projection$ denote the projection from $\Hspace_k$ to $\set{Z}^{\perp}_\anyfunction$, the orthogonal complement of $\set{Z}_\anyfunction$ in $\Hspace_k$. Defining $k\circ \anyfunction:\anyspace\times\anyspace\to\R$ by $k\circ\anyfunction(\anyelement,\anyelement') = k(\anyfunction(\anyelement),\anyfunction(\anyelement'))$, for $\anyelement,\anyelement'\in\anyspace$, we have the pullback $\Hspace_{k\circ\anyfunction}$ described as:
\begin{equation}
\Hspace_{k\circ\anyfunction} = \{f\circ\anyfunction \mid f \in \Hspace_k \}~,
\end{equation}
which is equipped with an inner product satisfying:
\begin{equation}
\inner{f\circ\anyfunction,g\circ\anyfunction}_{k\circ\anyfunction} = \inner{\projection f,\projection g}_k
\end{equation}
for all $f,g\in\Hspace_k$.
\end{lemma}

\thrmain*
\begin{proof}
Let $\measure{q}:\location\mapsto\pMeasure^\qry_\location$ denote the map from target to query location distribution. We can then define a kernel $\kp\circ\measure{q}(\location,\location'):=\kp(\measure{q}(\location),\measure{q}(\location'))=\kp(\pMeasure^\qry_\location,\pMeasure^\qry_{\location'})$, $\location,\location'\in\Sspace$. According to \autoref{thr:rkhs-composition}, the RKHS associated with $\kp\circ\measure{q}$ is given by:
\begin{equation}
\Hspace_{\kp\circ\measure{q}} = \left\lbrace \uncertain{g}\circ\measure{q} \,\middle|\, \uncertain{g} \in\Hspace_{\kp} \right\rbrace~,
\label{eq:proof-main-rkhs}
\end{equation}
equipped with an inner product whose associated norm is such that:
\begin{equation}
\norm{\uncertain{g}\circ\measure{q}}_{\kp\circ\measure{q}} = \norm{\projection\uncertain{g}}_{\kp} \leq \norm{\uncertain{g}}_{\kp}~,
\label{eq:proof-main-norms}
\end{equation}
for any $\uncertain{g}\in\Hspace_{\kp}$, where $\projection$ follows the definition in \autoref{thr:rkhs-composition}.

Considering the RKHS in \autoref{eq:proof-main-rkhs}, the result in \autoref{thr:main} follows after a few steps. Firstly, by \autoref{thr:exp-f}, for any $f\in\Hspace_k$, there is a unique $\uncertain{f}\in\Hspace_{\kp}$, such that:
\begin{equation}
\uncertain{f}\circ\measure{q}(\location)=\uncertain{f}(\pMeasure^\qry_\location) = \expectation[f(\random{\location})|\location]\,, \forall \location\in\Sspace ~.
\end{equation}
Then, letting $g := \uncertain{f}\circ\measure{q}$ and using $\kp\circ\measure{q}$ as the GP kernel, we apply \autoref{thr:sbo} to obtain a cumulative regret bound for $g$ as an objective, analogously to \autoref{sec:proof-bo}. From \autoref{eq:proof-main-norms} and \autoref{thr:exp-f}, we also have that:
\begin{equation}
\norm{g}_{\kp\circ\measure{Q}} \leq \norm{\uncertain{f}}_{\kp} = \norm{f}_k \leq \bound~.
\end{equation}
Lastly, to avoid needing an explicit formulation for $\uncertain{\mig}_{\iterIdx-1}^\qry$ to set $\beta_\iterIdx$, the known current information gain $\mutinfo{\observations_{\iterIdx-1};\vec{\uncertain{f}}_{\iterIdx-1} | \{\pMeasure^\qry_{\location_\primIdx}\}_{\primIdx=1}^{\iterIdx-1}}$ was instead used in the formulation of $\beta_\iterIdx$. This replacement maintains the same bounds obtained by \citet[Appendix C]{Chowdhury2017} and applied in \autoref{thr:sbo}.

For a given $\delta \in (0,1)$, \citeauthor{Chowdhury2017} arrive at the following result regarding a GP model with covariance $k:\Sspace\times\Sspace\to\R$ and any function $g\in\Hspace_{k}$:
\begin{multline}
\forall \iterIdx\geq 0,\forall\location\in\Sspace\,:
|\gpMean_\iterIdx(\location) - g(\location)| \leq\\
\sigma_\iterIdx(\location)\left(\bound+\sigma_\obsNoise\sqrt{2\log\frac{\sqrt{|(1+\factor)\eye+\mat{k}_\iterIdx}}{\delta}}\right)
\label{eq:proof-main-ucb}
\end{multline}
with probability greater than $1-\delta$, where we adjusted notation according to our setup. Observing that:
\begin{equation}
|(1+\factor)\eye + \mat{K}_\iterIdx| = |(\eye+(1+\factor)^{-1}\mat{K}_\iterIdx)||(1+\factor)\eye|~,
\end{equation}
the authors go on to show that:
\begin{equation}
\begin{split}
\log(|(1+\factor)\eye + \mat{K}_\iterIdx|) &= \log(|(\eye+(1+\factor)^{-1}\mat{K}_\iterIdx)|)\\&\quad+\iterIdx\log(1+\factor)\\&\leq 2\mig_\iterIdx + \factor\iterIdx~.
\end{split}
\label{eq:proof-main-mig}
\end{equation}
Choosing $\factor=2/\nIterations$ in the last result and replacing it into \autoref{eq:proof-main-ucb} leads to the formulation for $\beta_\iterIdx$ in \autoref{thr:sbo}. However, notice that:
\begin{equation}
\log(|(\eye+(1+\factor)^{-1}\mat{K}_\iterIdx)|) = 2\mutinfo{\observations_{\iterIdx};\vec{g}_\iterIdx|\{\location_{i}\}_{i=1}^\iterIdx}~.
\end{equation}
Using this identity in \autoref{eq:proof-main-mig} and replacing it into \autoref{eq:proof-main-ucb} yields the formulation for $\beta_\iterIdx$ in \autoref{thr:main}.

As in \autoref{sec:proof-bo}, the result in \autoref{thr:main} follows by noticing that the cumulative regret for $g$, as defined, is equivalent to the uncertain-inputs cumulative regret for $f$.
\end{proof}

\subsection{Proof of \autoref{thr:ig-iid}}
\label{sec:proof-ig-iid}
\thrigiid*
\begin{proof}
Let's consider the definitions of the information gain bounds. In the standard, deterministic-inputs case, the maximum information gain after $\nIterations$ iterations for a model $\gp(0,k)$ is given by:
\begin{equation}
\gamma_\nIterations = \max_{\locSet\subset\Sspace:|\locSet| = \nIterations} ~ \frac{1}{2} \log |\eye+\lambda^{-1}\mat{K}_\locSet| ~, 
\end{equation}
where $\mat{K}_\locSet = [k(\location,\location')]_{\location,\location'\in\locSet}$. In the case of $\gp(0,\kp)$ taking inputs from $\Pspace_\locNoise$, we have:
\begin{equation}
\uncertain{\gamma}_\nIterations(\Pspace_\locNoise) = \sup_{\pSet\subset\Pspace_\locNoise:|\pSet|=\nIterations}\frac{1}{2} \log |\eye+\lambda^{-1}\Kp_\pSet| ~, 
\end{equation}
where $\Kp_\pSet = [k(\pMeasure,\pMeasure')]_{\pMeasure,\pMeasure'\in\pSet}$. Both cases have the same parameter $\lambda>0$.

Considering the former definitions, observe that, if one can always find a set $\locSet\subset\Sspace$ that provides larger information gain than $\pSet$, for every choice of $\pSet\subset\Pspace_{\locNoise}$, $\gamma_\nIterations$ will then be larger than $\uncertain{\gamma}_\nIterations(\Pspace_\locNoise)$. The information gain depends on the determinants of the matrices $\eye+\lambda^{-1}\mat{K}_\locSet$ and $\eye+\lambda^{-1}\Kp_\pSet$, which is related to the positive-definiteness of both matrices.

A classic result in matrix analysis states that, if two $\nIterations$-by-$\nIterations$-matrices $\mat{A}$ and $\mat{B}$ are positive definite, and $\mat{A}-\mat{B}$ is positive semi-definite, their determinants satisfy $|\mat{A}|\geq|\mat{B}|$ \citep[see][Cor. 7.7.4]{Horn1985}. Recall that a matrix $\mat{A}\in\R^{\nIterations\times\nIterations}$ is positive semi-definite if and only if $\forall \vec{\alpha}\in\R^\nIterations,\, \vec{\alpha}^\transpose\mat{A}\vec{\alpha} \geq 0$, and positive definite if equality only holds for $\vec{\alpha}=\vec{0}$. Hence, we shall prove that:
\begin{multline}
\forall \{\pMeasure_i\}_{i=1}^\nIterations\subset\Pspace_\locNoise,\quad \exists \{\location_{i}\}_{i=1}^\nIterations\subset\Sspace:\\ \quad \vec{\alpha}^\transpose (\mat{K}_\nIterations-\Kp_\nIterations)\vec{\alpha}\geq 0, \quad \forall \vec{\alpha} \in \R^\nIterations~,
\label{eq:proof-ig-iid-condition}
\end{multline}
where $[\mat{K}_\nIterations]_{ij} = k(\location_i,\location_j)$ and $[\Kp_\nIterations]_{ij} = \kp(\pMeasure_i,\pMeasure_j)$, $i,j \in \{1,\dots,\nIterations\}$. For two positive semi-definite matrices $\mat{A},\mat{B}\in\R^{\nIterations\times\nIterations}$, let $\mat{A}\succcurlyeq\mat{B}$ denote that $\mat{A}-\mat{B}$ is positive semi-definite. Since:
\begin{equation}
\begin{split}
\mat{K}_\nIterations \succcurlyeq \Kp_\nIterations &\implies \eye+\lambda^{-1}\mat{K}_\nIterations \succcurlyeq \eye+\lambda^{-1}\Kp_\nIterations\\
&\implies |\eye+\lambda^{-1}\mat{K}_\nIterations| \geq |\eye+\lambda^{-1}\Kp_\nIterations|\\
&\implies \log|\eye+\lambda^{-1}\mat{K}_\nIterations| \geq \log|\eye+\lambda^{-1}\Kp_\nIterations|~,
\end{split}
\end{equation}
the condition in \autoref{eq:proof-ig-iid-condition}, if satisfied, then implies that $\mig_\nIterations \geq \uncertain{\mig}_\nIterations(\Pspace_\locNoise)$.

For a given $ \{\pMeasure_i\}_{i=1}^\nIterations\subset\Pspace_\locNoise$, define $\random{\location}_i \sim \pMeasure_i \in \Pspace_{\locNoise}$, for each $i\in\{1,\dots,\nIterations\}$. By the definition of $\Pspace_{\locNoise}$, we also have that each $\random{\location}_i = \locMean_i + \locNoise_i$, with $\locMean_i\in\Sspace$ and $\locNoise_i\sim \pMeasure_{\locNoise}$. Recall that, for any $\pMeasure,\pMeasure'\in\Pspace$, $\kp(\pMeasure,\pMeasure') = \inner{\meanMap_\pMeasure,\meanMap_{\pMeasure'}}_k$ and $\meanMap_\pMeasure = \expectation[k(\cdot,\random{\location})]$, $\random{\location}\sim\pMeasure$. Then we can write:
\begin{equation}
\begin{split}
\forall \vec{\alpha}\in\R^\nIterations,\quad \vec{\alpha}^\transpose\Kp_\nIterations\vec{\alpha} &= \sum_{i=1}^\nIterations\sum_{j=1}^\nIterations \alpha_i\alpha_j \kp(\pMeasure_i,\pMeasure_j)\\
&= \sum_{i=1}^\nIterations\sum_{j=1}^\nIterations \alpha_i\alpha_j \inner{\meanMap_{\pMeasure_i},\meanMap_{\pMeasure_j}}_k\\
&= \flexnorm{\sum_{i=1}^\nIterations\alpha_i\meanMap_{\pMeasure_i}}_k^2\\
&= \flexnorm{\sum_{i=1}^\nIterations\alpha_i \expectation\left[k(\cdot,\locMean_i+\locNoise_i)\right]}_k^2~.
\end{split}
\end{equation}
Now, as $\locNoise_i$ are \iid{} random variables, for any $\locNoise\sim\pMeasure_\locNoise$, it holds that:
\begin{equation}
\forall i \in \{1,\dots,\nIterations\},\quad \expectation[k(\cdot,\locMean_i+\locNoise_i)] = \expectation[k(\cdot,\locMean_i+\locNoise)]~.
\end{equation}
Applying this identity, we have that:
\begin{equation}
\begin{split}
\forall \vec{\alpha}\in\R^\nIterations,\, \vec{\alpha}^\transpose\Kp_\nIterations\vec{\alpha} &= \flexnorm{\expectation\left[\sum_{i=1}^\nIterations\alpha_i k(\cdot,\locMean_i+\locNoise_i)\right]}_k^2\\
&= \flexnorm{\expectation\left[\sum_{i=1}^\nIterations\alpha_i k(\cdot,\locMean_i+\locNoise)\right]}_k^2\\
&\leq \expectation\left[\flexnorm{\sum_{i=1}^\nIterations\alpha_i k(\cdot,\locMean_i+\locNoise)}_k^2\right]\\
&=\expectation\left[\sum_{i=1}^\nIterations\sum_{j=1}^\nIterations \alpha_i\alpha_j k(\locMean_i+\locNoise,\locMean_j+\locNoise)\right]\\
&=\sum_{i=1}^\nIterations\sum_{j=1}^\nIterations \alpha_i\alpha_j k(\locMean_i,\locMean_j)\\
&=\vec{\alpha}^\transpose\mat{K}_\nIterations\vec{\alpha}~,
\end{split}
\end{equation}
where the first inequality follows from the boundedness of the Bochner integral \citep[Ch. 2]{Mandrekar2015}, the fourth equality follows from $k$'s translation invariance, and $\mat{K}_\nIterations$ is defined by $[\mat{K}_\nIterations]_{ij} = k(\locMean_i,\locMean_j)$. Therefore, the set of mean locations $\{\locMean_i\}_{i=1}^\nIterations$ satisfies the condition in \autoref{eq:proof-ig-iid-condition}, leading to the result in \autoref{thr:ig-iid}, which concludes the proof.
\end{proof}

\subsection{Proof of \autoref{thr:approx}}
\label{sec:proof-approx}
The proof for \autoref{thr:approx} will make use of the following background. For further details, we refer the reader to \citet{Bauer1981} and \citet{Boucheron2013}.


\begin{definition}[Absolute continuity]
A measure $\measure{V}$ on a $\sigma$-algebra $\collection{X}$ is said to be absolutely continuous relative to a measure $\measure{U}$ on $\collection{X}$ if every $\measure{U}$-null set is also a $\measure{V}$-null set.
\end{definition}

Given a measure $\measure{m}$ on $\collection{X}$, a $\measure{m}$-null set is simply any set $\anyset\in\collection{X}$, such that $\measure{m}[\anyset] = 0$.

\begin{definition}[Kullback-Leibler divergence]
Let $\pMeasure$ and $\pMeasure'$ be two probability measures on $(\locDomain,\collection{X})$. The Kullback-Leibler divergence between the two measures is defined as:
\begin{equation}
\kl{\pMeasure'}{\pMeasure} := \int \log \frac{\diff\pMeasure'}{\diff\pMeasure}\diff\pMeasure'~,
\end{equation}
case $\pMeasure'$ is absolutely continuous relative to $\pMeasure$, or $\infty$ otherwise.
\end{definition}

\begin{lemma}[Pinsker's inequality]
\label{thr:pinsker}
Let $\pMeasure$ and $\pMeasure'$ be two probability measures on $(\locDomain,\collection{X})$, and let $\measure{u}$ be a common dominating measure of $\pMeasure$ and $\pMeasure'$. Assume that $\pMeasure'$ is absolutely continuous relative to $\pMeasure$. Then it holds that:
\begin{equation}
\int_{\locDomain} \left\lvert \pDensity(\location) - \pDensity'(\location) \right\rvert \diff\measure{u}(\location) \leq \sqrt{\frac{1}{2}\kl{\pMeasure'}{\pMeasure}}~,
\end{equation}
where $\pDensity = \frac{\diff\pMeasure}{\diff\measure{u}}$ and $\pDensity' = \frac{\diff\pMeasure'}{\diff\measure{u}}$ are the respective densities of each probability measure.
\end{lemma}

Now we can proceed to the proof of \autoref{thr:approx}, which is restated below.

\thrapprox*
\begin{proof}
\autoref{thr:approx} refers to the approximation error between the model $\Px$ and the actual distribution $\pMeasure_\location^\qry$ in terms of difference in the expected value of a function $f \in \Hspace_k$. The result simply follows by applying Pinsker's inequality (\autoref{thr:pinsker}).

For any $\iterIdx\geq 1$ and $\location\in\Sspace$, let $\hat{\pDensity}_\location$ and $\pDensity_\location$ denote the probability density functions of $\Px$ and $\pMeasure_\location^\qry$, respectively. Then we have that:
\begin{equation}
\begin{split}
|\expectation_{\pMeasure_\location^\qry}[f] - \expectation_{\Px}[f]| &= \left\lvert \int_\locDomain f(\location)(\pDensity_\location(\location')-\hat{\pDensity}_\location(\location'))\diff\location' \right\rvert\\
&\leq \norm{f}_\infty\int_\locDomain |\pDensity_\location(\location')-\hat{\pDensity}_\location(\location')|\diff\location'~.
\end{split}
\label{eq:proof-approx-bounds}
\end{equation}
Now note that, by the Cauchy-Schwartz inequality and $k$'s reproducing property, for any $\location \in \locDomain$:
\begin{equation}
\begin{split}
\norm{f}_\infty&=\sup_{\location\in\locDomain} |f(\location)|\\
&= \sup_{\location\in\locDomain} |\inner{f, k(\cdot,\location)}_{k}|\\
&\leq \sup_{\location\in\locDomain}\norm{f}_k \sqrt{k(\location,\location)}\\
&\leq \norm{f}_k~,
\end{split}
\label{eq:proof-approx-sup}
\end{equation}
since $k(\location,\location) \leq 1$ under our regularity assumptions. 

As both $\Px$ and $\pMeasure_\location^\qry$ are Gaussian measures, their support is the whole $\locDomain$, so that they are absolutely continuous with respect to each other. Then we can apply Pinsker's inequality to upper bound the remaining term in \autoref{eq:proof-approx-bounds}, which yields:
\begin{equation}
\begin{split}
\int_\locDomain |\pDensity_\location(\location')-\hat{\pDensity}_\location(\location')|\diff\location' &= \int_\locDomain |\hat{\pDensity}_\location(\location') - \pDensity_\location(\location')|\diff\location'\\
&\leq  \sqrt{\frac{1}{2}\kl{\pMeasure_\location^\qry}{\Px}}~.
\end{split}
\label{eq:proof-approx-pinsker}
\end{equation}
Plugging this result and the one in \autoref{eq:proof-approx-sup} back into \autoref{eq:proof-approx-bounds} yields:
\begin{multline}
\forall \iterIdx \geq 1,\,\forall\location\in\Sspace,\\ |\expectation_{\pMeasure_\location^\qry}[f] - \expectation_{\Px}[f]| \leq \norm{f}_k \sqrt{\frac{1}{2}\kl{\pMeasure_\location^\qry}{\Px}}~.
\label{eq:proof-approx-kl-bound}
\end{multline}

The  Kullback-Leibler divergence between two Gaussian distributions on $\R^\locDim$, $\pMeasure_\location^\qry$ and $\Px$, with covariance matrices as stated and same mean vectors is given by:
\begin{equation}
\kl{\pMeasure_\location^\qry}{\Px} = 
\frac{1}{2}\left(\tr(\mat{\hat{\Sigma}}^{-1}\mat{\Sigma}^\qry) - \locDim + \log\frac{|\mat{\hat{\Sigma}}|}{|\mat{\Sigma}^\qry|}\right)~,
\label{eq:app-g-kl}
\end{equation}
which comes from a known result \citep[p. 203]{Rasmussen2006} and the fact that:
\begin{equation}
\tr(\mat{\hat{\Sigma}}^{-1}(\mat{\Sigma}^\qry-\mat{\hat{\Sigma}})) = \tr(\mat{\hat{\Sigma}}^{-1}\mat{\Sigma}^\qry-\eye) = \tr(\mat{\hat{\Sigma}}^{-1}\mat{\Sigma}^\qry) - \locDim~.
\end{equation}
Replacing \autoref{eq:app-g-kl} into \autoref{eq:proof-approx-kl-bound} yields the result in \autoref{thr:approx}.
\end{proof}

\subsection{The uncertain-inputs squared-exponential kernel}
\label{sec:k_use}
Here we present the formulation for the uncertain-inputs squared exponential kernel when both inputs follow a Gaussian distribution. This formulation is the analytical solution for Equation \ref{eq:ugp-k} under these settings, and is also found in \citet[Eq. 3.53]{Girard2004}. Here we present it as follows:
\begin{multline}
\label{eq:k_use}
\kp(\normal(\locMean,\mat{\Sigma}),\normal(\locMean',\mat{\Sigma}')) =\\
\frac{
\sigma_f^2 \exp\left(-\frac{1}{2} (\locMean-\locMean')^\transpose (\mathbf{W} + \mat{\Sigma} + \mat{\Sigma}')^{-1} (\locMean - \locMean') \right)
}{
|\mathbf{I} + \mathbf{W}^{-1}(\mat{\Sigma} + \mat{\Sigma}')|^{1/2}
}~,
\end{multline}
where $\sigma^2_f$ is a signal variance parameter, set to 1 in our experiments, and $\mathbf{W}$ is a diagonal squared length-scales matrix. We used Equation \ref{eq:k_use} to implement the GP covariance function for uGP-UCB in the experiments, while the other methods were configured with the deterministic-inputs squared-exponential kernel.

\bibliography{references}

\end{document}